\documentclass[journal,10pt,twocolumn]{IEEEtran}
\usepackage{epsfig}
\usepackage{graphicx}
\usepackage{subfigure}

\newtheorem{theorem}{Theorem}

\newtheorem{lemma}{Lemma}

\newtheorem{definition}{Definition}
\newtheorem{assumption}{Assumption}
\newtheorem{algorithm}{Algorithm}
\newtheorem{remark}{Remark}

\usepackage{amssymb}
\usepackage{makeidx}
\usepackage{amsmath}
\usepackage{graphicx}
\usepackage{epsf}
\usepackage{epsfig}
\usepackage{psfig}
\usepackage{ccaption}
\usepackage{array}
\usepackage{tabularx}
\usepackage{multirow}
\usepackage{epsfig}
\usepackage{cite}
\usepackage{procedure,algorithm,algorithmic}

\begin{document}

\title{Analysis on the Nonlinear Dynamics of Deep Neural Networks: Topological Entropy and Chaos}
\author{Husheng Li
\thanks{H. Li is with the Department of Electrical Engineering and Computer Science, the University of Tennessee, Knoxville, TN (email: husheng@eecs.utk.edu). This work was supported by the National
Science Foundation under grants CNS-1525226, CNS-1525418, CNS-1543830.}}
\maketitle

\begin{abstract}
The theoretical explanation for the great success of deep neural network (DNN) is still an open problem. In this paper DNN is considered as a discrete-time dynamical system due to its layered structure. The complexity provided by the nonlinearity in the DNN dynamics is analyzed in terms of topological entropy and chaos characterized by Lyapunov exponents. The properties revealed for the dynamics of DNN are applied to analyze the corresponding capabilities of classification and generalization. In particular, for both the hyperbolic tangent function and the rectified linear units (ReLU), the Lyapunov exponents are both positive given proper DNN parameters, which implies the chaotic behavior of the dynamics. Moreover, the Vapnik-Chervonenkis (VC) dimension of DNN is also analyzed, based on the layered and recursive structure. The conclusions from the viewpoint of dynamical systems are expected to open a new dimension for the understanding of DNN.
\end{abstract}

\section{Introduction}
Deep neural network (DNN, a.k.a. deep learning) is the most shining star in the community of information science in the last decade \cite{Goodfellow2016}. It has brought a paradigm shift to artificial intelligence and many cross-disciplinary areas. Despite its great success in applications, the theoretical explanation and mathematical framework are still open problems for DNN. There have been substantial efforts on the mathematical analysis for DNN, such as using the wavelet framework to understand DNN \cite{Mallat2016}, applying the framework of function approximation \cite{Liang2017}, enumerating the linear regions after the nonlinear mapping of DNN \cite{Montufar2014}, analyzing the shape deformation in the transient chaos of DNN \cite{Poole2016}, and information theoretic analysis \cite{Huang2017}, to name a few. Although these excellent studies have made significant progress on a deeper understanding of DNN, they are still insufficient to fully describe the behavior, quantitatively analyze the performance and thus systematically design DNNs. 

In this paper, we propose a systematic framework to analyze DNN by considering it as a dynamical system. A prominent feature of DNN is the layered structure, usually having many hidden layers (e.g., hundreds of layers). Each layer is a mapping having similar structures, namely a linear mapping followed by a nonlinear operation. One can consider each layer as a nonlinear transformation for the input vector. Therefore, the DNN with the multi-layer structure can be considered as a discrete-time nonlinear dynamical system, where the initial value is the input vector to the DNN. The transformation is specified by the DNN parameters, and the final value of the transformations can be easily handled using simple classifiers such as linear ones. 

The view of DNN as a dynamical system brings a novel approach for the theoretical study, due to the powerful framework of dynamical systems developed in the past 100 years. On one hand, a DNN should be sufficient complex to classify different objects (e.g., images). The topological entropy, developed in the study of dynamical systems in 1960s \cite{Downarowicz2011}, can be used to characterize the complexity of dynamics. On the other hand, for complex objects, a DNN may need to distinguish two feature vectors that are very close to each other, pulling them apart for the final linear classification. This implies that the behavior of DNN near these points needs to be chaotic; i.e., the distance of the two points will become significant after the nonlinear mappings even though they are initially very close to each other. This can also be understood from the viewpoint of classification surface. For the original inputs of different classes, the classification surface could be very complicated; however, the DNN can deform the complicated surface to a simple hyperplane. Therefore, the inverse transformation of DNN (e.g., linear transformation plus $\tanh^{-1}$) should be able to warp a hyperplane to a very complex surface, which is illustrated in Fig. \ref{fig:paper}. The nonlinear warping requires chaotic behaviors characterized by Lyapunov exponents. 

\begin{figure}
  \centering
  \includegraphics[scale=0.34]{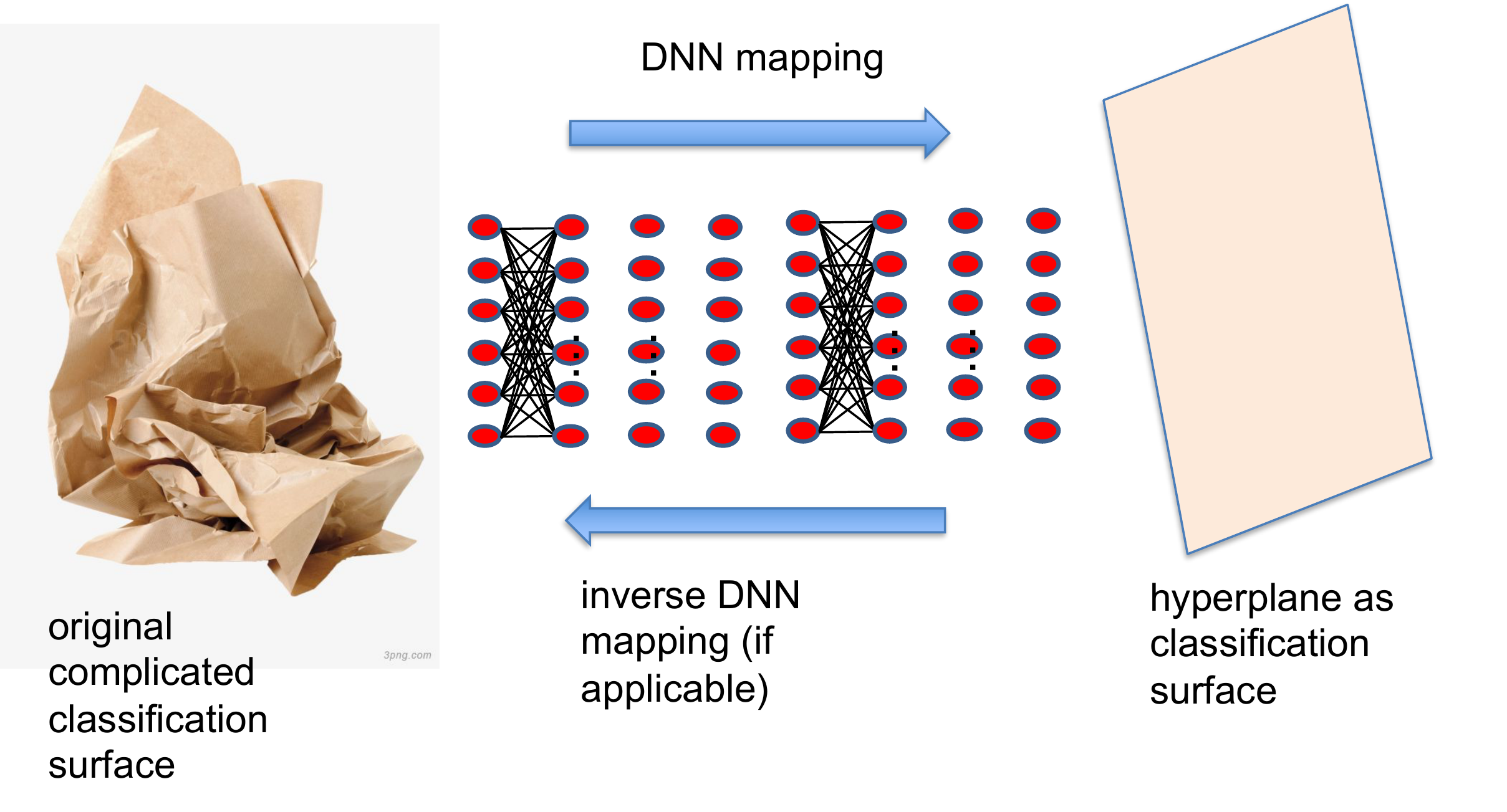}
  \caption{An illustration of the deformation of classification surface by DNN.}\label{fig:paper}
\end{figure}

The analysis of dynamical behaviors in DNN can be applied to analyze various properties of DNN. On one hand, the complexity incurred by the nonlinearity in the dynamics of DNN can be used to quantify the classification capability of DNN, and predict the required number of layers. On the other hand, if the nonlinear dynamics provide too much complexity, the Vapnik-Chervonenkis (VC) dimension could be high and thus affect the generalization capability. Meanwhile, the layered structure of DNN also provides a layer-by-layer analysis for the VC-dimension. These applications will be discussed in this paper.

The remainder of this paper is organized as follows. The system model of DNN is introduced in Section \ref{sec:model}. The topological entropy and the chaotic behavior, characterized by the Lyapunov exponent, of DNN are analyzed in Section \ref{sec:entropy}. The applications of the nonlinear dynamics behavior in the capabilities of classification and generalization of DNN are discussed in Section \ref{sec:application}. The conclusions are provided in Section \ref{sec:conclusion}.

\section{System Model}\label{sec:model}
In this section, we introduce the system model of DNN. We first explain the structure of DNN. Then, we model a DNN as a random dynamical system for facilitating the analysis.

\subsection{Structure of DNN}
For simplicity of analysis, we consider the problem of binary classification, which outputs a binary label in $\{-1,1\}$ given a $d$-dimensional input feature vector $\mathbf{x}\in \mathbb{R}^d$. We assume that the feature selection algorithm has been fixed and is beyond the scope of this paper. For simplicity of analysis, we assume that all the features are within the compact set $[-1,1]^d$. 

We assume that there are $D$ hidden layers in the DNN. The weights of layer $i$ are denoted by $\mathbf{W}_i$, which is a $d\times d$ matrix. Note that we assume that each $\mathbf{W}_i$ is a square matrix and all the vectors $\{\mathbf{x}_i\}_{i=1,...,D}$ have the same dimension. This is to simplify the analysis and can be extended to more generic cases. The input to layer $i$ is denoted by $\mathbf{x}_{i-1}$, $i=1,...,D$. At layer $i$, the output $\mathbf{x}_i$ (also the input of layer $i+1$) is given by
\begin{eqnarray}\label{eq:dynamics}
\left\{
\begin{array}{ll}
&\mathbf{h}_i=\mathbf{W}_i\mathbf{x}_{i-1}+\mathbf{b}_i\\
&\mathbf{x}_{i}=\phi(\mathbf{h}_i)
\end{array}
\right.,
\end{eqnarray}
where $\phi$ is the elementwise nonlinear mapping from the intermediate result $\mathbf{h}_i$ to the input $\mathbf{x}_i$ for the next layer. In this paper, we consider the following two cases: 
\begin{itemize}
\item Hyperbolic tangent $\tanh$: This operation is used in traditional neural networks. It confines the intermediate vectors $\mathbf{x}_{i}$ within $[-1,1]^d$. The advantage of $\tanh$ is the smoothness, which facilitates the analysis, and the compactness of the state space.

\item Rectified linear units (ReLU) $x=\max\{h,0\}$: ReLU is demonstrated to outperform the hyperbolic tangent $\tanh$ in recent progress of DNNs. The challenge of ReLU is the loss of differentiability at point $x=0$. We can assume that the derivative is a piecewise continuous function:
\begin{eqnarray}
\frac{d\phi(h)}{dh}=\left\{
\begin{array}{ll}
0,\qquad h<0\\
1,\qquad h>0
\end{array}
\right.,
\end{eqnarray}
where the derivative is not defined at $x=0$.
\end{itemize}

The $d$-dimensional output of the last hidden layer is fed into a linear classifier with coefficients $\mathbf{w}$ and offset $b$. The final decision is given by
\begin{eqnarray}
\mbox{decision}=\mbox{sign}(\mathbf{w}^T\mathbf{x}_D+b).
\end{eqnarray}
The overall structure of DNN studied in this paper is summarized in Fig. \ref{fig:NN}. Note that we do not consider more specific structures for the DNN in this paper. For example, more structures are embedded into the convolutional neural networks (CNNs), recurrent neural networks (RNNs) and so on, which improve the performance over the generic structure of DNNs. It will be our future work to incorporate these special structures into the dynamical systems based study. 

\begin{figure}
  \centering
  \includegraphics[scale=0.4]{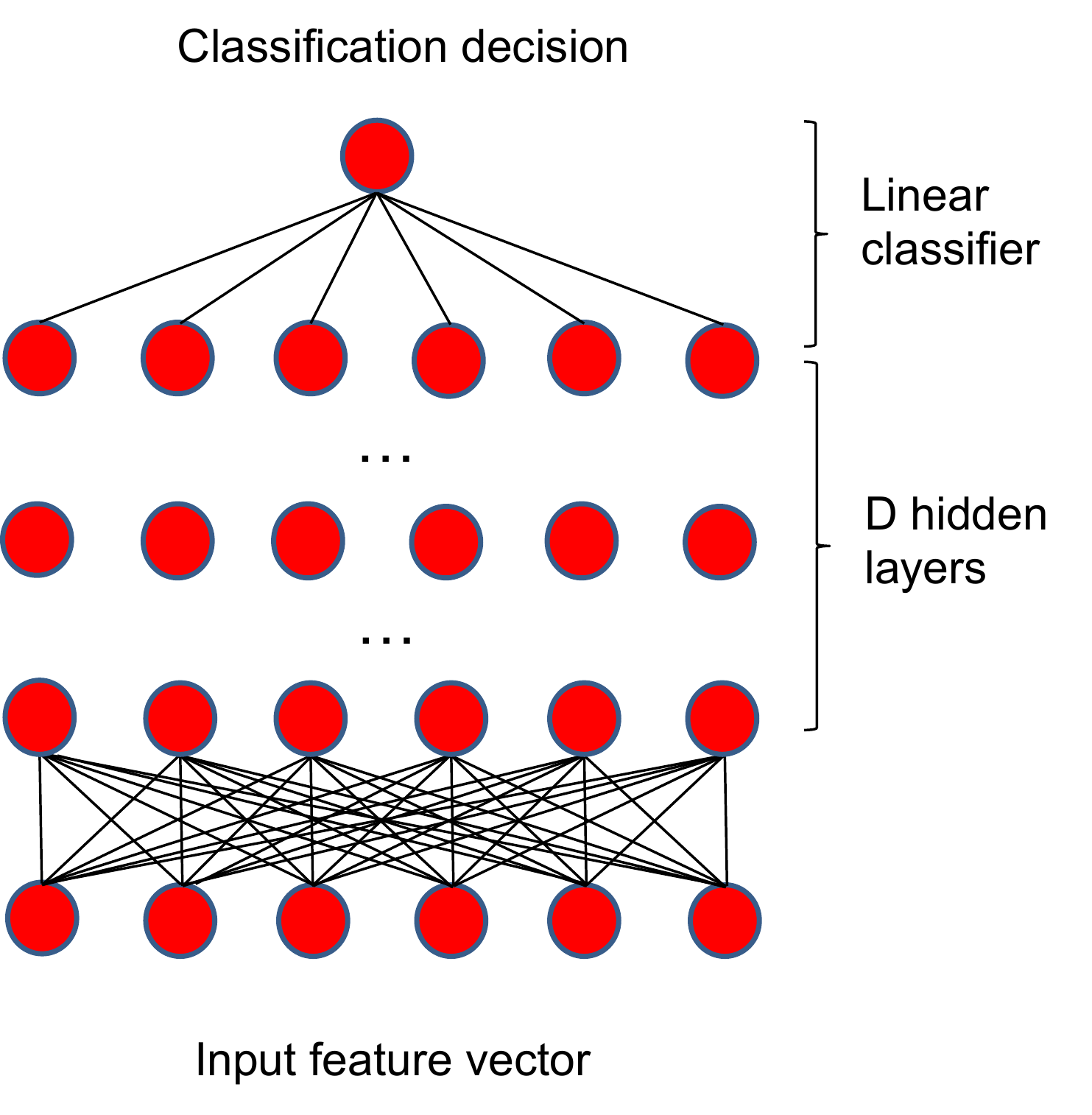}
  \caption{Deep neural network for classification.}\label{fig:NN}
\end{figure}

\subsection{Random Dynamical System Modeling}
As justified in the introduction, we can consider a DNN as a dynamical system. However, in traditional dynamical systems, the mapping of the system state, denoted by $T$, is identical throughout the dynamics, while the mappings in different layers of DNN are typically different. Therefore, we use the random dynamical system \cite{Arnold1998} to model the mappings of different layers in DNN. A random dynamical system consists of a series of possibly distinct transformations $T_1,...,T_n, ...$, namely 
\begin{eqnarray}
\mathbf{x}_i=T_{i}(\mathbf{x}_{i-1}).
\end{eqnarray}
In the context of DNN, $T_{i}$ is determined by $(\mathbf{W}_i,\mathbf{b}_i,\phi)$. A rigorous definition of random dynamical system is given in Appendix \ref{appdx:random}.

\section{Topological Entropy and Chaos}\label{sec:entropy}
In this section, we discuss the topological entropy and chaos, characterized by the Lyapunov exponent, of DNN. We first define these quantities in dynamical systems. Then, we analyze the chaotic behaviors in DNN.

\subsection{Topological Entropy and Lyapunov Exponent}

\subsubsection{Topological Entropy}
We first define the topological entropy in a space $X$ equipped with metric $m(\cdot,\cdot)$ \cite{Downarowicz2011,Matveev2009}. Consider a mapping $T:X\rightarrow X$. For time $n\in \mathbb{N}$, a metric $m^n$ between two points in $X$ with respect to $T$ is defined as
\begin{eqnarray}\label{eq:metric_n}
m^n(x_1,x_2)=\max_{i=1,...,n}m\left(T^{i-1}(x_1),T^{i-1}(x_2)\right).
\end{eqnarray}
We define the $n$-th order ball of point $x$, denoted as $B^n(x,\epsilon)$, as the ball around $x$ with radius $\epsilon$ under the definition of metric in (\ref{eq:metric_n}). A set of points in $X$, denoted by $S$, are said to be an $(n,\epsilon)$-spanning orbit set, if for any $x\in X$, we can always find a point $y\in S$ such that $m^n(x,y)\leq \epsilon$. Then we define
\begin{eqnarray}
r(T,n,\epsilon)=\min_{S \mbox{ be $(n,\epsilon)$-spanning set}}|S|.
\end{eqnarray}
Then we define the spanning orbits based topological entropy as follows, where the subscript $s$ means `spanning'.
\begin{definition}[\cite{Downarowicz2011,Matveev2009}]\label{def:entropy007}
The spanning orbits based topological entropy is defined as
$h_{s}(T)=\lim_{\epsilon\rightarrow 0}h_s(T,\epsilon),$
where $h_{s}(T,\epsilon)=\lim_{n\rightarrow \infty}\frac{1}{n}H_s(T,n,\epsilon)$, and $H_s(T,n,\epsilon)=\log_2 r(T,n,\epsilon)$.
\end{definition}
A detailed introduction to alternative definitions of topological entropy can be found in Appendix \ref{appdx:topo}.

Note that the above definition is for traditional dynamical systems. In random dynamical systems, the topological entropy can be defined in the same manner as that in Definitions \ref{def:entropy007}, simply by replacing $T^{i}$ with $T_i\circ T_{i-1}\circ...\circ T_1$.

\subsubsection{Ensemble Topological Entropy}
The conventional topological entropy measures the number of distinct paths (up to diminishing errors) in the dynamical system. In this paper we extend it to the topological entropy of an ensemble of dynamics. For an integer $L>0$, we define a $dL$-dimensional hyper-vector $\mathbf{z}=(\mathbf{x}^1,...,\mathbf{x}^L)$, where each $\mathbf{x}^i$ is a $d$-vector. Then, for the sequence of transformations $\{T_t\}_{t=1,2,...}$ that map from $\mathbb{R}^d$ to $\mathbb{R}^d$, we define the transformation $\tilde{T}^L_t:\mathbb{R}^{dL}\rightarrow\mathbb{R}^{dL}$ as
\begin{eqnarray}\label{eq:ensemble}
\tilde{T}^L_t(\mathbf{z})=(T_t(\mathbf{x}^1),...,T_t(\mathbf{x}^L)).
\end{eqnarray}
Consider $M$ sets of transformations $\{T_{mt}\}_{m=1,...,M,t=1,...,n}$, we define $r_e(\{T_{mt}\}_{m,t},n,\epsilon,\mathbf{z})$ as the number of distinct paths (within distance $\epsilon$ in (\ref{eq:metric_n})) beginning from $\mathbf{z}$ under the $M$ composite transformations $\{\tilde{T}^L_{mt}\}_{m=1,...,M,t=1,...,n}$. We define $H_e(\{T_{mt}\}_{m,t},n,\epsilon,\mathbf{z})=\frac{1}{n}\log_2r_e(\{T_{mt}\}_{m,t},n,\epsilon,\mathbf{z})$. We can further define $h_e(\{T_{mt}\}_{m,t},\mathbf{z})$ by letting $n\rightarrow\infty$ and $\epsilon\rightarrow 0$, similarly to the topological entropy $h_s$. As will be seen, the concept of ensemble topological entropy will be used in the calculation of VC-dimension of DNNs. Unfortunately we are still unable to evaluate the ensemble topological entropy analytically, which will be our future study.

\subsubsection{Lyapunov Exponent} As metrics characterizing the chaotic behavior of dynamical systems, the Lyapunov spectrum \cite{Arnold1998} quantifies how fast two nearby points depart from each other as the system state evolves with time, as illustrated in Fig. \ref{fig:chao}. Considering an initial point $\mathbf{x}_0$, we define
\begin{eqnarray}
\mathbf{J}_t(\mathbf{x}_0)=\mathbf{J}_{\mathbf{x}}(T_t\circ T_{t-1} \circ ... \circ T_1(\mathbf{x}))|_{\mathbf{x}=\mathbf{x_0}},
\end{eqnarray}
where $\mathbf{J}$ is the Jacobian matrix defined as $\mathbf{J}_{\mathbf{x}}(\mathbf{y})=\left(\frac{\partial y_i}{\partial {x}_j}\right)_{ij}$. Obviously $\mathbf{J}_t(\mathbf{x}_0)$ measures the sensitivity to the perturbation on the initial value $\mathbf{x}_0$. Then, we define
\begin{eqnarray}
\mathbf{L}(\mathbf{x}_0)=\lim_{t\rightarrow\infty} (\mathbf{J}_t(\mathbf{x}_0)\mathbf{J}_t^T(\mathbf{x}_0))^{\frac{1}{t}},
\end{eqnarray}
whose convergence is guaranteed by the Multiplicative Ergodic Theorem proved by V. I. Oseledec \cite{Oseledets1968}. Then, we define the $i$-th largest Lyapunov exponent $\lambda_i$ as
\begin{eqnarray}
\lambda_i(\mathbf{x}_0)=\log_2 \Lambda_i(\mathbf{x}_0),
\end{eqnarray}
where $\Lambda_i(\mathbf{x}_0)$ is the $i$-th largest eigenvalue (obviously nonnegative) of $\mathbf{L}(\mathbf{x}_0)$. The set of eigenvalues $\{\lambda_1,...,\lambda_n\}$, sorted in a descending order, is called the Lyapunov spectrum of the random dynamical system $\{T_1,T_2,...,T_i,...\}$ \cite{Arnold1998}. Note that, for the generic case, the Lyapunov spectrum is dependent on the initial value $\mathbf{x}_0$. For ergodic dynamics, the Laypunov spectrum is identical for all initial values, since each point in the state space will be visited sooner or later.

\begin{figure}
  \centering
  \includegraphics[scale=0.4]{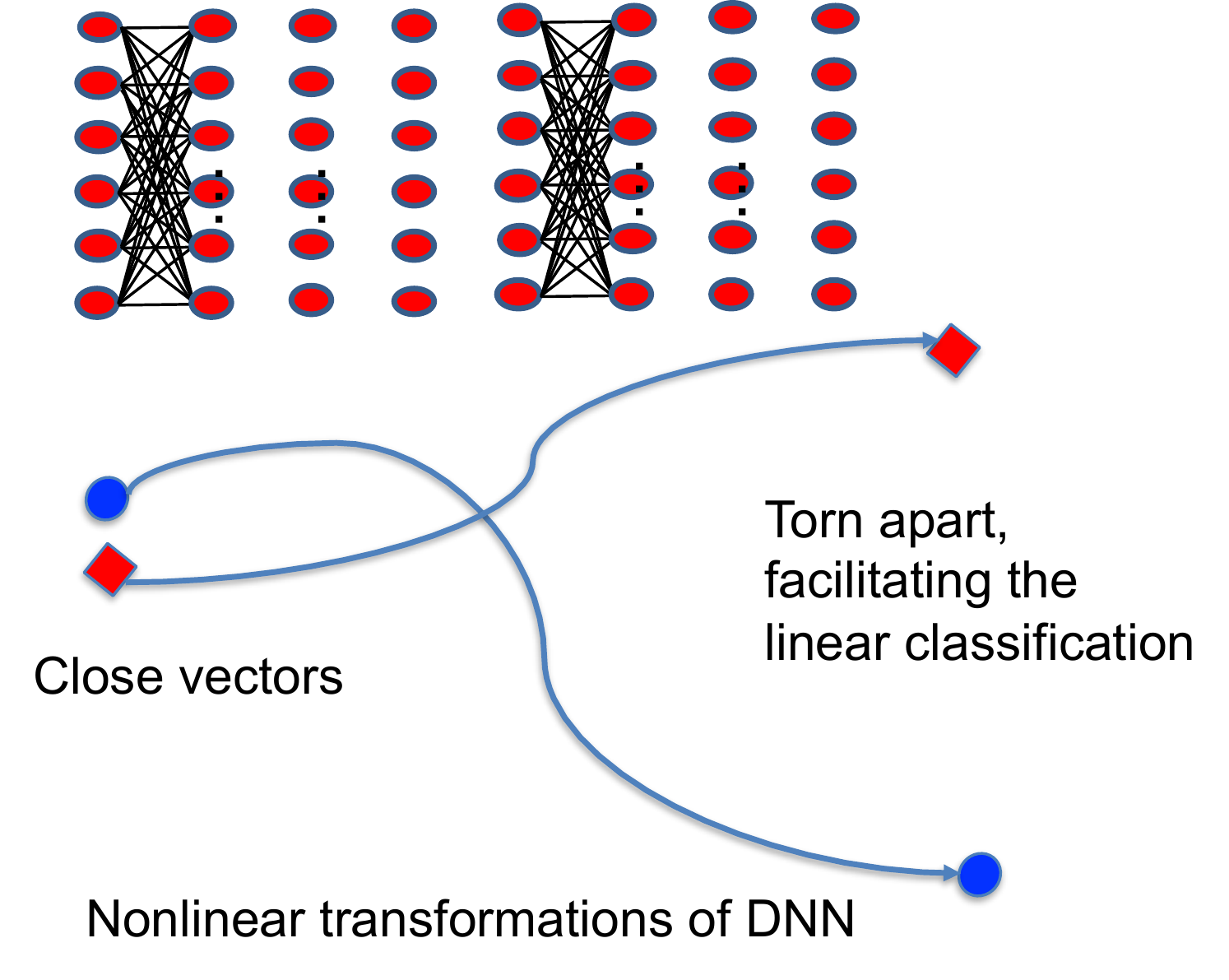}
  \caption{The illustration of the chaotic behavior in DNN.}\label{fig:chao}
\end{figure}

For the random dynamics in DNN, we have the following lemma characterizing the Lyapunov spectrum, which can be easily shown by calculus given the structure of DNN:
\begin{lemma}
For the random dynamics in DNN, we have
\begin{eqnarray}
\mathbf{J}_t(\mathbf{x}_0)=\prod_{i=1}^t \left( \mathbf{S}_i(\mathbf{x}_i)\mathbf{W}_i\right)
\end{eqnarray}
where when $\phi$ is the hyperbolic tangent $\tanh$ we have
\begin{eqnarray}
\mathbf{S}_i(\mathbf{x}_i)=diag\left(1-\tanh^2\left(x_{ij}\right)\right)_{j=1,...,n},
\end{eqnarray}
and when $\phi$ is the ReLU function we have
\begin{eqnarray}
\mathbf{S}_i(\mathbf{x}_i)=diag\left(sign\left(x_{ij}\right)\right)_{j=1,...,n},
\end{eqnarray}
where $sign(x)$ equals 1 if $x>0$ and 0 if $x<0$.
\end{lemma}

\subsubsection{Relationship}
The topological entropy can be derived from the Laypunov spectrum through the following theorem (as a corollary of Theorem A in \cite{Ledrappier1988}):
\begin{theorem}[\cite{Ledrappier1988}]\label{thm:lay_top}
Suppose that the random dynamics with transformations $T_1$, $T_2$, ..., are ergodic. Then, we have
\begin{eqnarray}
h_s(T_1,T_2,...)=\sum_{i=1}^n \max\{\lambda_i,0\},
\end{eqnarray}
where $\lambda_i$ is identical for all initial points due to the ergodicity.
\end{theorem}
\begin{remark}
For non-ergodic dynamics, we also have similar conclusions \cite{Ledrappier1988}. 
We can calculate the topological entropy by evaluating the Lyapunov spectrum. In the following discussion, we will study the Lyapunov spectrum at different points in the state space, and thus obtain the conclusions on the topological entropy of DNNs.
\end{remark}

\subsection{Local Chaos with Deterministic Mapping} The Lyapunov spectrum could be dependent on the initial value if the dynamics are non-ergodic. Hence, the system could be chaotic at certain points, while being non-chaotic in the remainder of the state space. The following theorems show that the maximum Lyapunov exponent of the DNN dynamics could be arbitrarily large. This implies that a DNN, as a dynamical system, could be highly chaotic at certain regions, thus being able to form highly complicated classification surfaces. 

\subsubsection{Hyperbolic Tangent Function}
The following theorem shows the existence of chaotic behavior when the DNN uses the hyperbolic tangent function. The proof is given in Appendix \ref{appdx:DNN_Lya}.

\begin{theorem}\label{thm:local_Lya}
Consider the hyperbolic tangent function $\phi=\tanh$. There exist parameters of DNN and inputs (namely the initial vector of the dynamical system) such that the maximum Lyapunov exponent $\lambda_1$ of the corresponding random dynamics is arbitrarily large. 
\end{theorem}

\subsubsection{ReLU Function}
For the ReLU function case, the challenge is the unboundedness of the state space (although many practical algorithms add constraints on the parameters such as normalization). It has been shown that an $n$-dimensional linear dynamics $\mathbf{x}(t+1)=\mathbf{A}\mathbf{x}(t)$ has a topological entropy of $\sum_{i=1}^n\log_2 \max(\left|\lambda_i(\mathbf{A})\right|,1)$, where $\lambda_i(\mathbf{A})$ is the $i$-th eigenvalue of the matrix $\mathbf{A}$ \cite{Matveev2009}. Therefore, when any eigenvalue of $\mathbf{A}$ has a norm greater than 1, the topological entropy is positive. However, this does not imply the complexity of the dynamics, since it is still linear. The positive topological entropy simply stems from the enlarging norm of system state in the unbounded state space. Therefore, to characterize the essential complexity and chaotic behavior of the DNN dynamics subject to ReLU function, we need to consider a compact space, instead of the original unbounded state space. 

To this end, we consider the angle of the system states in the state space, and omit the norm of the system states. This is reasonable, since the linear separability of the system state samples at the DNN output is determined by the angles, instead of the norms. The following theorem shows the existence of chaotic behavior in the angles of the system states for the 2-dimensional case. The proof is given in Appendix \ref{appdx:DNN_Lya2}. Note that the 2-dimensional case can be easily extended to more generic cases. 

\begin{theorem}\label{thm:local_Lya2}
Consider a 2-dimensional DNN with the ReLU function, where the system state is $\mathbf{x}(t)=(x_1(t),x_2(t))$. The angle of the system state is denoted by $\theta(t)=\arctan \left(\frac{x_2(t)}{x_1(t)}\right)$. Then, for any number $C>0$, one can always find a set of parameters of the DNN and a time $T$ such that
\begin{eqnarray}
\left\|\nabla_{\mathbf{x}(0)}\rho(\mathbf{x}(T))\right\|>C.
\end{eqnarray}
\end{theorem}

\subsection{Random Mapping}
In the above analysis, the linear mappings are designed specifically to result in chaos. It demonstrates only that the nonlinearity brought by the operation $\tanh$ may result in chaos. However, the dynamics in real DNNs may not be the same as that in Theorem \ref{thm:local_Lya}. Therefore, we study the ensemble of elementwisely Gaussian distributed matrices in this subsection. Note that it is reasonable to assume that the elements in the matrices $\{W_i\}_{i=1,...,D}$ are normally distributed \cite{Mallat2016,Poole2016}. Therefore, the Gaussian ensemble of linear mappings is a good approximation to real DNNs. For simplicity, we analyze the hyperbolic tangent function $\tanh$. The analysis on the ReLU function is more difficult due to the challenge of analyzing the angle dynamics that are highly nonlinear. It will be our future research to extend the conclusions on $\tanh$ to ReLU functions. 
 
Due to the difficulty of mathematical analysis, we first carry out numerical calculations for the maximum Lyapunov exponent $\lambda_1$ in the random mapping. In the numerical calculation, we assume $\phi=\tanh$ and that the elements in $\mathbf{W}_t$ are i.i.d. Gaussian random variables with zero expectation and variance $\sigma^2$. We set $\mathbf{b}=0$. The corresponding Lyapunov exponents, as a function of the variance $\sigma^2$ and dimension $d$, are shown in Fig. \ref{fig:vector1}. The contour graph is given in Fig. \ref{fig:vector2}.  We observe that, fixing the variance $\sigma^2$, the higher the dimension $d$ is, the more possible that the Lyapunov exponent becomes positive. Fixing the dimension $d$, a too small $\sigma^2$ or a too large $\sigma^2$ may make the Lyapunov exponent negative. In particular, when $d=1$, namely the dynamics is scalar, the maximum Lyapunov exponent is always negative, which implies that there be no chaotic behavior in the scalar dynamics. Since usually the input dimension $d$ is very high in real DNN applications, the results imply that chaotic behaviors do exist in practical DNNs.

\begin{figure}
  \centering
  \includegraphics[scale=0.35]{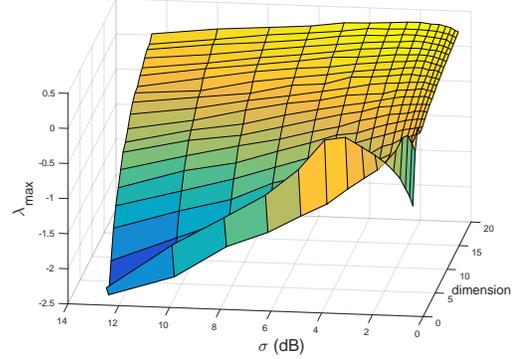}
  \caption{Lyapunov exponent $\lambda_1$ for vector deterministic dynamics with Gaussian matrices.}\label{fig:vector1}
\end{figure}

\begin{figure}
  \centering
  \includegraphics[scale=0.4]{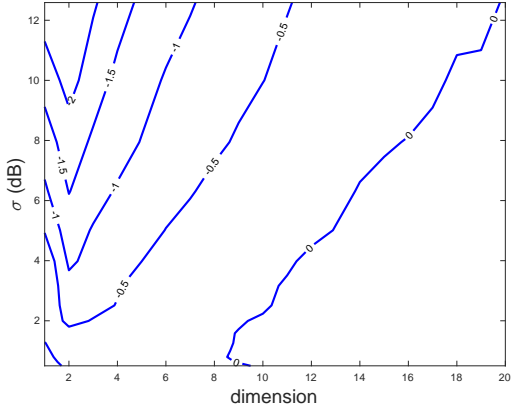}
  \caption{Contour of Lyapunov exponent for vector deterministic dynamics with Gaussian matrices.}\label{fig:vector2}
\end{figure}

The analysis on the Lyapunov spectrum of the random dynamics incurred by the Gaussian ensemble of mappings is difficult. The challenge is that there are two series of dynamics in the DNN, namely that of the tangent vector $\mathbf{u}_t$:
\begin{eqnarray}\label{eq:dyn_u}
\mathbf{u}_t=\mathbf{S}_t(\mathbf{x}_t)\mathbf{W}_t\mathbf{u}_{t-1},
\end{eqnarray}
and that of the system state $\mathbf{x}_t$:
\begin{eqnarray}
\mathbf{x}_t=\phi\left(\mathbf{W}_t\mathbf{x}_{t-1}+\mathbf{b}_t\right).
\end{eqnarray}

In Theorem \ref{thm:local_Lya}, we design the mappings such that the dynamics of $\mathbf{u}_t$ and $\mathbf{x}_t$ are decoupled. However, in the case of Gaussian random mappings, the two dynamics are coupled, since the matrix $\mathbf{S}_t$ is determined by $\mathbf{x}_t$. In order to obtain an analytic result, we apply the approach of mean field analysis in statistical mechanics by fixing the impact of $\mathbf{x}_t$ on $\mathbf{u}_t$; i.e., we consider the average behaviors of the random variables. In more details, we approximate the dynamics of the tangent vector $\mathbf{u}_t$ in $\mathbf{u}_t=\mathbf{S}_t(\mathbf{x}_t)\mathbf{W}_t\mathbf{u}_{t-1}$ using 
\begin{eqnarray}
\beta_t=\frac{E\left[\|E\left[\mathbf{S}_t(\mathbf{x}_t)\right]\mathbf{W}_t\mathbf{u}_{t-1}\|^2\right]}{E\left[\|\mathbf{u}_{t-1}\|^2\right]},
\end{eqnarray}
where the numerator is the norm of the tangent at time $t$ subject to the expectation of $\mathbf{S}_t$ while the denominator is the norm of the system state at time $t-1$.
The metric $\beta_t$ approximately characterizes the enlarging rate of the norm of $\|\mathbf{u}_t\|$ in a deterministic manner. In particular, taking expectation on $\mathbf{S}_t$ is to simplify the analysis. Then, we have the following assumption of the mean field dynamics:
\begin{assumption}\label{assum:mean}
We have the following assumptions for the dynamics of DNN with random linear mappings, whose justification for the assumption will be given in Appendix \ref{appx:global}.
\begin{itemize}
\item When $\beta_t>1$ for all sufficiently large $t$, the norm of $\mathbf{u}_t$ is expected to increase exponentially, thus resulting in chaotic behavior. 
\item When $t$ and $n$ are sufficiently large, we have $\sum_{i=1}^nx_{ti}^2=C$, where $C$ is a time-invariant constant.
\end{itemize}
\end{assumption}

Given Assumption \ref{assum:mean}, we have the following theorem providing a condition for the chaotic behavior with random linear mappings. The proof of is given in Appendix \ref{appx:global}.

\begin{theorem}\label{thm:global}
Given Assumption \ref{assum:mean}, we have $\lambda_1>1$ for the DNN dynamics with Gaussian random linear mappings, if 
\begin{eqnarray}\label{eq:chaos_ineq}
\left(1-4\tanh \left(\sigma\sqrt{\frac{2}{\pi}}h\left(\sigma\sqrt{\frac{2}{\pi}}\right)\right)\right)d\sigma^2>1,
\end{eqnarray}
where $h(x)$ is the nonzero and positive solution to the equation $z=\tanh(x z)$ (where $z$ is the unknown while $x$ is the parameter).
\end{theorem}
\begin{remark}
The inequality (\ref{eq:chaos_ineq}) is based on mean dynamics analysis and Assumption \ref{assum:mean}. However, it shows that a higher dimensional dynamics is more possible to be chaotic. Moreover, a too large or too small $\sigma^2$ disables (\ref{eq:chaos_ineq}). Both coincide the observations in the numerical results. An explicit expression for the function $h$ in terms of infinite series, based on the Lagrange Inverse Formula \cite{Bruijn2010} can also be found in Appendix \ref{appx:global}.
\end{remark}

\section{Capabilities of Classification and Generalization}\label{sec:application}
In this section, we apply the metrics characterizing the dynamics of DNN to the analysis of classification and generalization capabilities.

\subsection{Classification Capability}\label{sec:classification}
We first define the complexity of classification problem and then relate it to the topological entropy of DNN dynamics.

\subsubsection{Complexity of Classification} We define the complexity of classification, which facilitates the subsequent analysis of DNN classification. The philosophy is to measure how difficult it is to separate the samples from different classes. There have been some studies on how to define the complexity of pattern recognition \cite{Basu2006,Ho2000,Hoekstra1996,Lorena2015}. However, they are mainly focused on the nonlinearity of classifier and the overlap of features, which is significantly different from our proposed definition.

For defining the complexity of classification, we first need the following definition of set radius.
\begin{definition}
For a connected and compact set $A$ in $\mathbb{R}^d$, we define its radius as
$r(A)=\sup_{\mathbf{x},\mathbf{y}}\{\|\mathbf{x}-\mathbf{y}\||\mathbf{x},\mathbf{y}\in A\}$.
Then, we define the radius of a partition $P$ of $A$ as $r(P)=\max_{A}\{r(A)|A\in P\}.$
\end{definition}

Then, we define the classification complexity as follows:
\begin{definition}\label{def:complexity0}
Given the sample set $\mathbf{X}$, consisting of $\mathbf{X}_1$ labeled by +1 and $\mathbf{X}_2$ labeled by -1, we say that a Borel set $A$ is hybrid with respect to $\mathbf{X}$ if $A\cap \mathbf{X}_j\neq \phi$, $j=1,2$. Suppose that $\mathbf{X}$ is contained in a compact set $A$ (e.g., the convex hull of $\mathbf{X}$). For a partition $P$ of $A$, we define
\begin{eqnarray}\label{eq:hybrid}
N(P|\mathbf{X})=\left|\{A|A\in P, A\mbox{ is hybrid w.r.t. $\mathbf{X}$}\}\right|.
\end{eqnarray}
Then, we define the $\epsilon$-complexity given the sample set $\mathbf{X}$ as
\begin{eqnarray}\label{eq:complexity}
C(\mathbf{X},\epsilon)=\log_2 \inf_{P:r(P)\leq \epsilon}N(P|\mathbf{X}).
\end{eqnarray}
 \end{definition}

\begin{remark}
The complexity of classification in (\ref{eq:complexity}) characterizes how the two classes of samples are mixed with each other. It also characterizes the size of the boundary surface between the two classes; the larger, the more complex. As will be seen, the definition of classification complexity will facilitate the analysis based on topological entropy. The links between the classification complexity and related metrics are discussed in Appendix \ref{appdx:related}.
\end{remark}

\subsubsection{Number of DNN Layers}
Now, we apply the concept of topological entropy to the analysis of DNN, in terms of the capability of classification. 
We need the following definition:
\begin{definition}
Consider two sets of vectors $Z_1=\{\mathbf{z}_{11},...,\mathbf{z}_{1N_1}\}$ and $Z_1=\{\mathbf{z}_{11},...,\mathbf{z}_{1N_1}\}$. We say that $Z_1$ and $Z_2$ are affine separable with an $\epsilon$-margin, if there exists an affine classifier $(\mathbf{w},b)$ (where $\|\mathbf{w}\|=1$) such that 
\begin{eqnarray}
\left\{
\begin{array}{ll}
&\mathbf{w}^T\mathbf{z}_1-b\geq \frac{\epsilon}{2}\\
&\mathbf{w}^T\mathbf{z}_2-b\leq -\frac{\epsilon}{2}
\end{array}
\right..
\end{eqnarray}
\end{definition}

Then, we have the following lower bound on the number of layers in DNN in the following theorem, whose proof is given in Appendix \ref{appdx:layer_num}:
\begin{theorem}\label{thm:layer_num}
Suppose that the output of DNN hidden layers is affine separable with $\epsilon$-margin for the sample set $\mathbf{X}$. Then, the number of layers $D$ is lower bounded by
\begin{eqnarray}
D\geq \frac{C(\mathbf{X},\epsilon)+1}{H_s(\{T_n\}_{n=1,..,D},D,\epsilon)}.
\end{eqnarray}
\end{theorem}
\begin{remark}
Recall that $H_s(\{T_n\}_{n=1,..,D},D,\epsilon)$ is defined in Definition \ref{def:entropy007}. Note that, when $D$ is sufficiently large and $\epsilon$ is sufficiently small, $H_s(\{T_n\}_{n=1,..,D},D,\epsilon)$ can be well approximated by $h_s(\{T_n\}_{n=1,2,...})$, which becomes independent of $D$ and $\epsilon$.
\end{remark}

\subsection{Generalization Capability}\label{sec:generalization}
A over-complex model for classification may fail in the generalization to samples beyond the training set. Hence, it is important to study the generalization capability of DNNs, characterized by the VC-dimension \cite{Anthony1999}. Note that the VC-dimension is usually studied by quantifying the separated components in the parameter space \cite{Anthony1999}. Meanwhile, the VC-dimension of DNN has been substantially explored in \cite{Bartlett2017}, which is slightly different from the setup in our study, since it addresses piecewisely linear DNNs. Our results may not be as precise as that in \cite{Bartlett2017}; however, we leverage the layered structure of DNN and may provide a new direction for studying the VC-dimension of DNN.
Note that the generalization capability of DNN has been shown to be better than predicted by VC-dimension argument in typical applications \cite{Zhang2017}. Many explanations have been given to the excellent generalization capability of DNN, e.g. \cite{Dinh2017}. However, our study shows a new recursive approach to explore the generalization of DNN.

\subsubsection{VC-dimension}
It is well known that the generalizability of a hypothesis set $H$ is determined by its VC-dimension, which is defined as follows:
\begin{definition}
We say that a set of $m(\geq 1)$ points is fully shattered by a binary hypothesis set $H$ if the hypotheses in $H$ can realize all the $2^m$ possible labelings for $S$. The VC-dimension is defined as
\begin{eqnarray}
d_{VC}(H)=\max\{|S|: S \text{ is fully shattered by }H\}.
\end{eqnarray}
\end{definition}
The relationship between the VC-dimension and the generalization capability can be found in \cite{Hanneke2016}.

\subsubsection{VC-dimension in DNN}
In the subsequent analysis, we find both lower and upper bounds for the VC-dimension from the view point of dynamical systems.
Based on the ensemble topological entropy, we have the following upper bound for the VC-dimension of DNN. The proof is given in Appendix \ref{appdx:vc_bound}.
\begin{theorem}\label{thm:vc_bound}
The VC-dimension of DNN is upper bounded by
\begin{small}
\begin{eqnarray}\label{eq:VC_dim}
d_{VC}\leq \sup_{\{T_{mt}\}_{m,t},\mathbf{z}}H_e(\{T_{mt}\}_{m=1,...,2^{d_{VC}},t=1,...,D},D,\epsilon,\mathbf{z})D,
\end{eqnarray}
\end{small}
when the DNN output is affine separable with $\epsilon$ margin.
\end{theorem}
\begin{remark}
Again the quantity $H_e$ can be approximated by $h_e$ when $D$ is sufficiently large and $\epsilon$ is sufficiently small.
The above upper bound for the VC-dimension of DNN is obtained from the complexity analysis. Unfortunately, we are still unable to numerically or analytically evaluate the ensemble topological entropies $H_e$ or $h_e$.
\end{remark}

We can also analyze the details of shattering the samples by leveraging the nonlinearity of $\tanh$ in a layer-by-layer manner, thus resulting in the following lower bound of VC-dimension. The proof is given in Appendix \ref{appdx:vc_lbound}.

\begin{theorem}\label{thm:vc_lbound}
The number of distinct samples that can be shattered by a DNN is at least $D+3$, where $D$ is the number of layers. Therefore, the VC-dimension of DNN is lower bounded by
$d_{VC}\geq D+3.$
\end{theorem}
\begin{remark}
The conclusion shows that the VC-dimension increases at least linearly with the number of layers. The layer-by-layer analysis may be extended to find finer bounds, since we only analyzed the 2-dimensional case in the proof. 
\end{remark}

\section{Numerical Results}
In this section, we use the ResNet-101 classifier, together with the image set called \textit{ImageNet}, to evaluate the largest Lyapunov exponents for practical DNNs and real images. For notational simplicity, we call the largest Lyapunov exponent the Lyapunov exponent. 

\subsection{ResNet-101}
We used the pretrained ResNet-101 \cite{He2016}, which is a convolutional deep neural network and is available in Matlab codes. The DNN of ResNet-101 has 101 layers. The input and output of ResNet-101 are images and the corresponding image categories. The coefficients of the DNN are trained using the celebrated ImageNet database \cite{ImageNet}. The evaluation of the Lyapunov exponents is based on the image set `Caltech 101', which contains 101 classes of images. Each image is normalized to the size of $224\times 224$. In total 8733 images are used in our numerical computations. 

\subsection{Estimation of Lyapunov Exponents}
For traditional dynamical systems, the Lyapunov exponents can be estimated using the approaches in \cite{Parker1989} (pp.71--73). However, these algorithms are designed for dynamics having identical mappings in each round, while the DNN dynamics are random, and have different nonlinear mappings in each layer due to different coefficients. Moreover, the number of layers of ReSNet-101 (i.e., 101 layers) may not be sufficiently large such that the maximum increasing direction emerges spontaneously from any initial direction. Therefore, we propose the algorithm in Procedure \ref{alg:Lya} to approximately estimate the Lyapunov exponents.

\begin{algorithm}[H]
	\caption{Algorithm for Estimating Lyapunov Exponents}\label{alg:Lya}
	\begin{algorithmic}[1]
		\STATE{Read an image $\mathbf{x}(0)$ from the image set `Caltech 101'.}
		\STATE{Calculate the output $\mathbf{y}$ of DNN given the input $\mathbf{x}(0)$.}
		\FOR{t =1:presetTimes}
		         \STATE{Randomly generate a perturbation $\delta \mathbf{x}$ with sufficiently small magnitude.}
		         \STATE{Calculate the perturbed input $\mathbf{x}_t(0)=\mathbf{x}(0)+\delta \mathbf{x}$.}
		         \STATE{Calculate the output $\mathbf{y}_t$ of DNN given the input $\mathbf{x}_t(0)$.}
		         \STATE{Calculate the approximate derivative $d_t=\frac{\|\mathbf{y}-\mathbf{y}_t\|}{\|\mathbf{x}(0)-\mathbf{x}_t(0)\|}$.}
		\ENDFOR
		\STATE{Choose the maximum absolute value from the set $\{d_1,d_2,...\}$ as the approximate Lyapunov exponent.}
	\end{algorithmic}
\end{algorithm}

\subsection{Numerical Results}
Consider each image as a point in the `image space', whose dimension is 50176. Therefore, the Lyapunov exponent, estimated using Procedure \ref{alg:Lya}, can be considered as the corresponding Lyapunov exponent at the corresponding point. 

\subsubsection{Trajectory}
In Fig. \ref{fig:trajectory}, the trajectories beginning from an image point and a perturbed point are plotted. Since the trajectories are in a very high dimensional space, we plot only the trajectories on two dimensions. We observe that, although the perturbation is very small, the two trajectories are significantly different. Note that the trajectories return to the origin frequently, which is due to the function of ReLU. 

\begin{figure}
  \centering
  \includegraphics[scale=0.28]{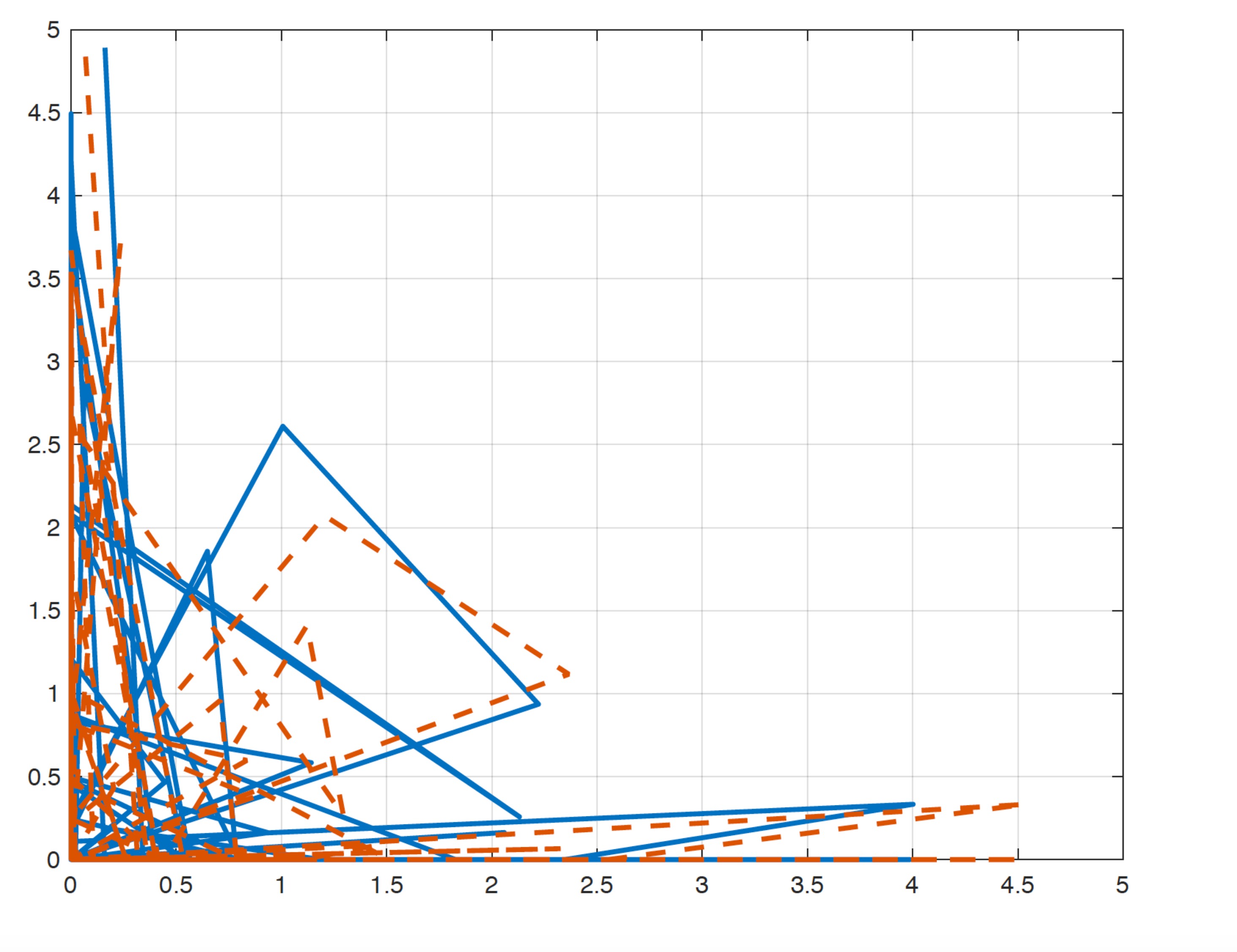}
  \caption{The projection of trajectory subject to perturbation.}\label{fig:trajectory}
\end{figure}

\subsubsection{Histogram of Lyapunov Exponents}

\begin{figure}
  \centering
  \includegraphics[scale=0.3]{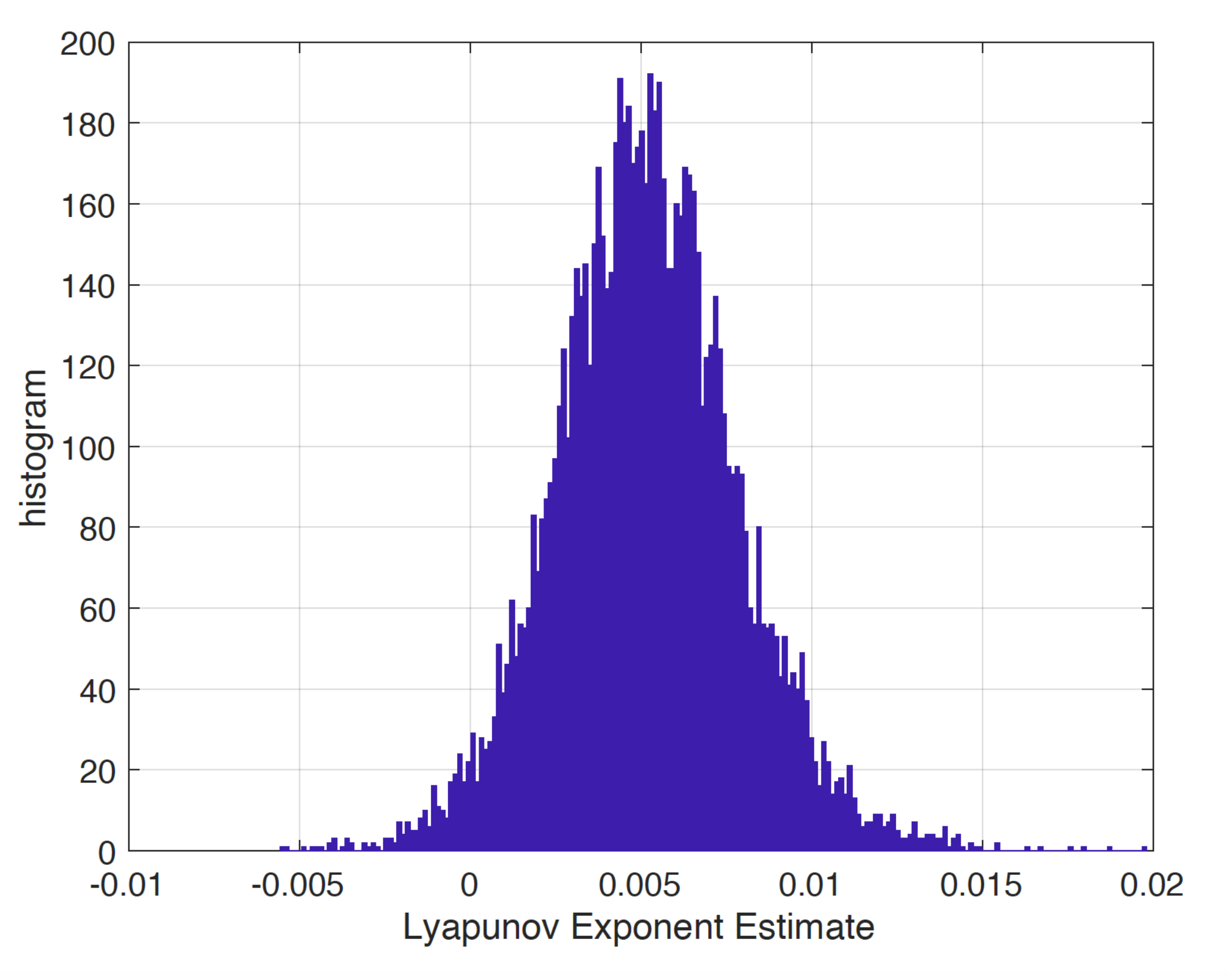}
  \caption{Histogram of Lyapunov exponents.}\label{fig:histogram}
\end{figure}

The histogram of Lyapunov exponents obtained from the 8733 images and the algorithm in Procedure \ref{alg:Lya} is given in Fig. \ref{fig:histogram}. We observe that the distribution has a bell shape and is concentrated around the value of 0.005. Some are negative, which means that the perturbation diminishes. We notice that the Lyapunov exponents at very few points are close to 0.2. The images of the highest 4 Lyapunov exponents are given in Fig. \ref{fig:top4}. It is not clear why these 4 images have large Lyapunov exponents. However, the large Lyapunov exponents imply that the DNN classifier is substantially more sensitive at these images, which may imply that they are closer to other classes of images and are thus more difficult to classify. This possible implication will be tested in our future research. 

\begin{figure}
  \centering
  \includegraphics[scale=0.36]{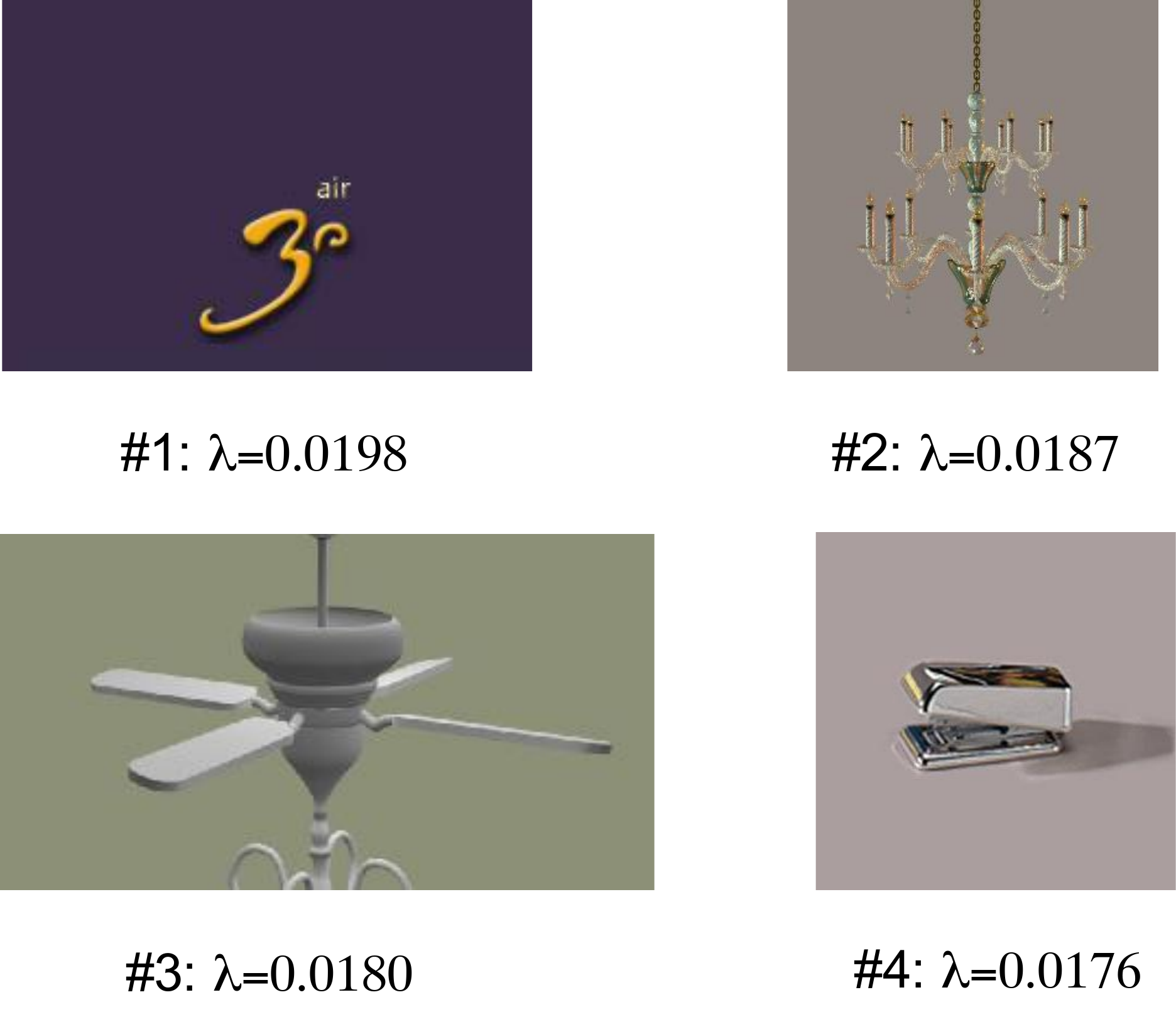}
  \caption{The four images having the greatest Lyapunov exponents.}\label{fig:top4}
\end{figure}

\subsubsection{Comparison of Different Classes}

We also plot the expectation and variance of the Lyapunov exponents for the 101 classes in the plane, as shown in Fig. \ref{fig:mean_var}. The four classes of images having the greatest variance of Lyapunov exponents are shown in Fig. \ref{fig:classes}. We notice that the variances of different image classes are around 6e-6, while the overall variance is 7e-6. Therefore, Lyapunov exponents are chaotically distributed among different classes of images. 

\begin{figure}
  \centering
  \includegraphics[scale=0.3]{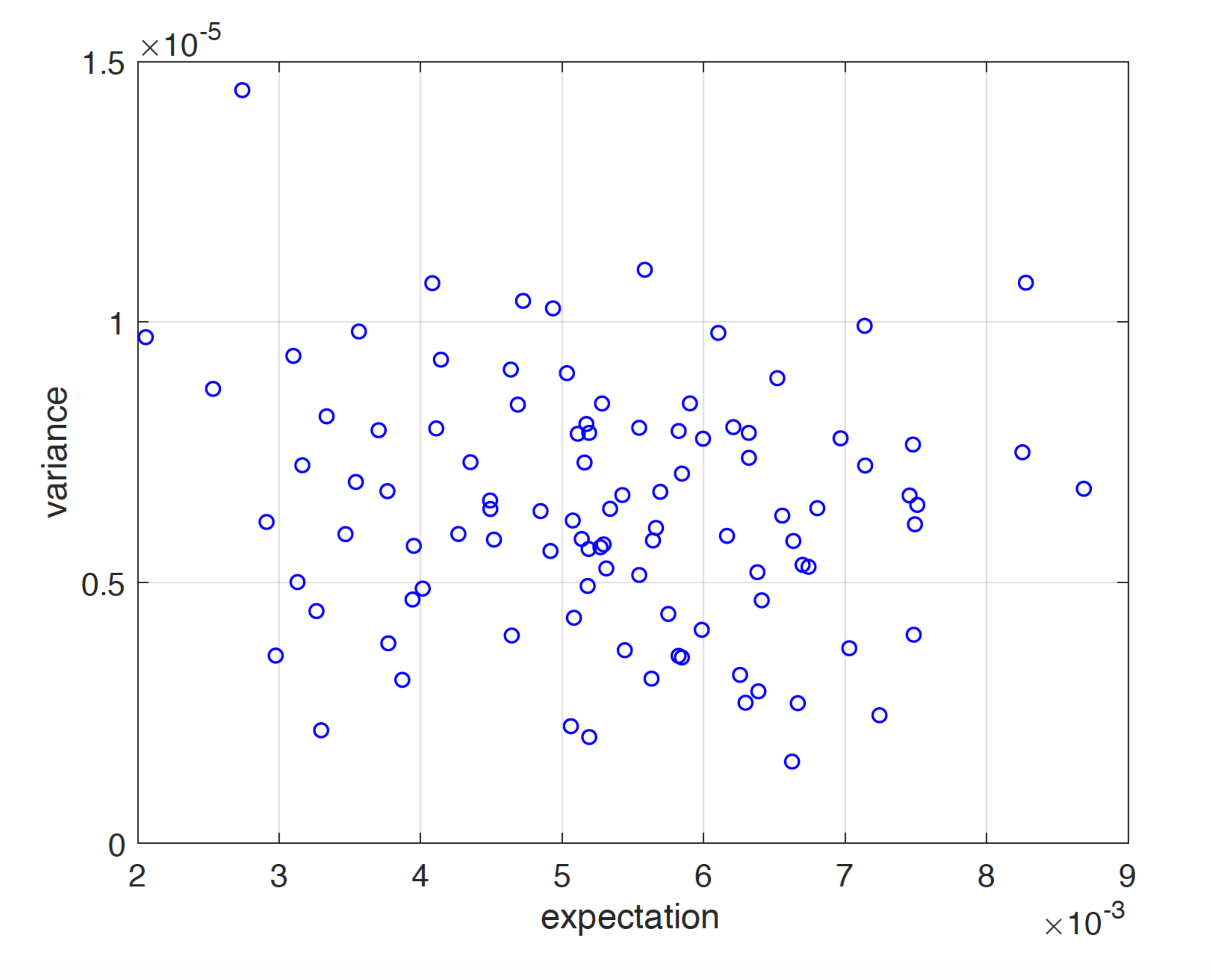}
  \caption{The distribution of mean and variance of different classes.}\label{fig:mean_var}
\end{figure}

\begin{figure}
  \centering
  \includegraphics[scale=0.36]{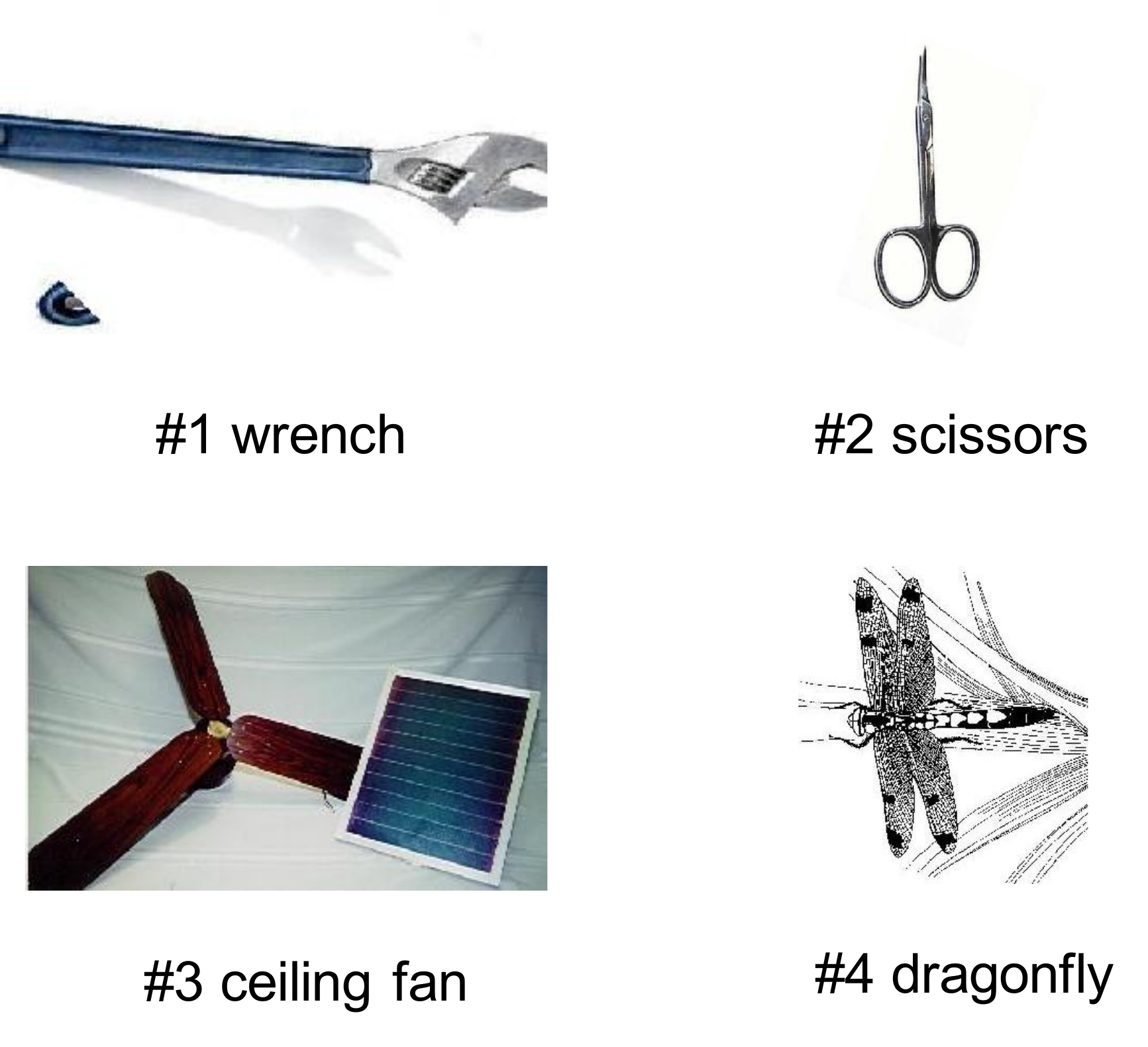}
  \caption{The top 4 classes of images having the highest expectations of Lyapunov exponents.}\label{fig:classes}
\end{figure}

\section{Conclusions}\label{sec:conclusion}
In this paper, we have proposed to consider DNN as a random dynamical system, and thus study the properties of DNNs from the viewpoint of dynamical systems. To characterize the complexity of the DNN dynamical system, we have analyzed the corresponding topological entropy and Lyapunov exponents. In particular, we have demonstrated that there may exist chaotic behavior due to the nonlinearity of DNN, which explains why DNN can form complicated classification surfaces. In our future study, we will calculate the topological entropy either analytically (e.g., in the high dimensional regime) or numerically (e.g., using the generating function of symbolic dynamics).

\appendices
 
\section{Random Dynamical Systems}\label{appdx:random}
Here we provide a rigorous definition for random dynamical systems \cite{Arnold1998}. A measurable random dynamical system is a 4-tuple $(\Omega,\mathcal{F},\mathbb{P},(\theta(t))_{t\in \mathbb{T}})$ over a measurable space $(X,\mathcal{B})$, where $X$ is the system state set and $\mathcal{B}$ is the $\sigma$-algebra:
\begin{itemize}
\item $\Omega$: the sample space representing the random events in the random dynamical system;
\item $\mathcal{F}$: the $\sigma$-algebra of $\Omega$;
\item $\mathbb{P}$: the distribution over $\Omega$;
\item $\mathbb{T}$: the set of time index;
\item $\theta$: the time shift acting on the random event $w$. 
\end{itemize}
The random dynamical system provides a mapping $\psi$:
\begin{eqnarray}
\psi:\mathbb{T}\times \Omega\times X\rightarrow X,
\end{eqnarray}
namely mapping the system state $x\in X$ at time $t$ to $\psi(t,w,x)$, given the realization $w\in \Omega$. 

The dynamical system $(\Omega,\mathcal{F},\mathbb{P},(\theta(t))_{t\in \mathbb{T}}))$ is called ergodic, if all sets in $\mathcal{I}$ have probabilities 0 or 1, where $\mathcal{I}\subset \mathcal{F}$ is the sub-$\sigma$-algebra formed by all invariant sets. Note that a set $A$ is called invariant with respect to $\theta$ if $\theta^{-1}A=A$. 

\section{Definitions of Topological Entropy}\label{appdx:topo}
There are three equivalent definitions for the topological entropy \cite{Downarowicz2011}. We will introduce them separately in the subsequent discussion. For the three definitions, we consider a topological dynamical system represented by a triple $(X,T,S)$. Here $X$ is a metric space, in which the metric is denoted by $d$. As will be seen, the metric is not necessary for one of the definitions. However, we still assume the structure of metric space, since it will be used in the first two definitions. We assume that $X$ is compact. When the space $X$ is not compact, an unbounded dynamical system such as $x(t)=2x(t-1)$ may have positive topological entropy but is still not complex. We define the transformation $T$ mapping from $X$ to $X$. Moreover, we assume that $T$ is continuous; i.e., $T^{-1}$ maps from open sets to open sets in the topological space $X$.

\subsection{Spanning Orbit Based Definition}
We first define the topological entropy based on spanning orbits, namely a set of typical orbits of the dynamics, such that for any orbit one can always find a sufficiently close one in this set. To this end, we need to define the distance between two orbits, namely the metric $d^n$ defined as
\begin{eqnarray}
d^n(x,y)=\max\left\{d\left(T^ix,T^iy\right),i=0,...,n-1\right\},
\end{eqnarray}
for orbits with initial points $x$ and $y$, repectively. 
Intuitively speaking, the distance between the two orbits is the maximum of the distances at all time slots.
We define the $n$-th order ball of point $x\in X$ as the ball around $x$ with radius $\epsilon$ and denote it by $B^n(x,\epsilon)$.
Fig. \ref{fig:Husheng_capacity_spnnaing} shows that a track beginning from $x_2$ is within $B^n(x_1,\epsilon)$. 

Now we define the `typical set' of orbits (or equivalently the initial points). We say a set $F$ of points in $X$ is $(n,\epsilon)$-spanning if it intersects every $(n,\epsilon)$-ball (i.e., intersecting $B^n(x,\epsilon)$ for every $x$). Intuitively speaking, this means that, for any point $x$ in $X$ (or equivalently the orbit beginning from $x$), one can always find a sufficiently close orbit (up to the error $\epsilon$), or the corresponding initial point, to approximate the orbit beginning from $x$. 

Then, we define
\begin{eqnarray}
r(n,\epsilon)=\min\left\{\# F: F\mbox{ is an }(n,\epsilon)-\mbox{spanning}\right\},
\end{eqnarray}
namely the cardinality of the minimum $(n,\epsilon)$-spanning. One can consider each point or the corresponding orbit in the set as a `codeword', in the context of source coding in communications. Then, $r(n,\epsilon)$ is the minimum number of codewords to represent all possible orbits, with errors up to $\epsilon$. 

 \begin{figure}
  \centering
  \includegraphics[scale=0.5]{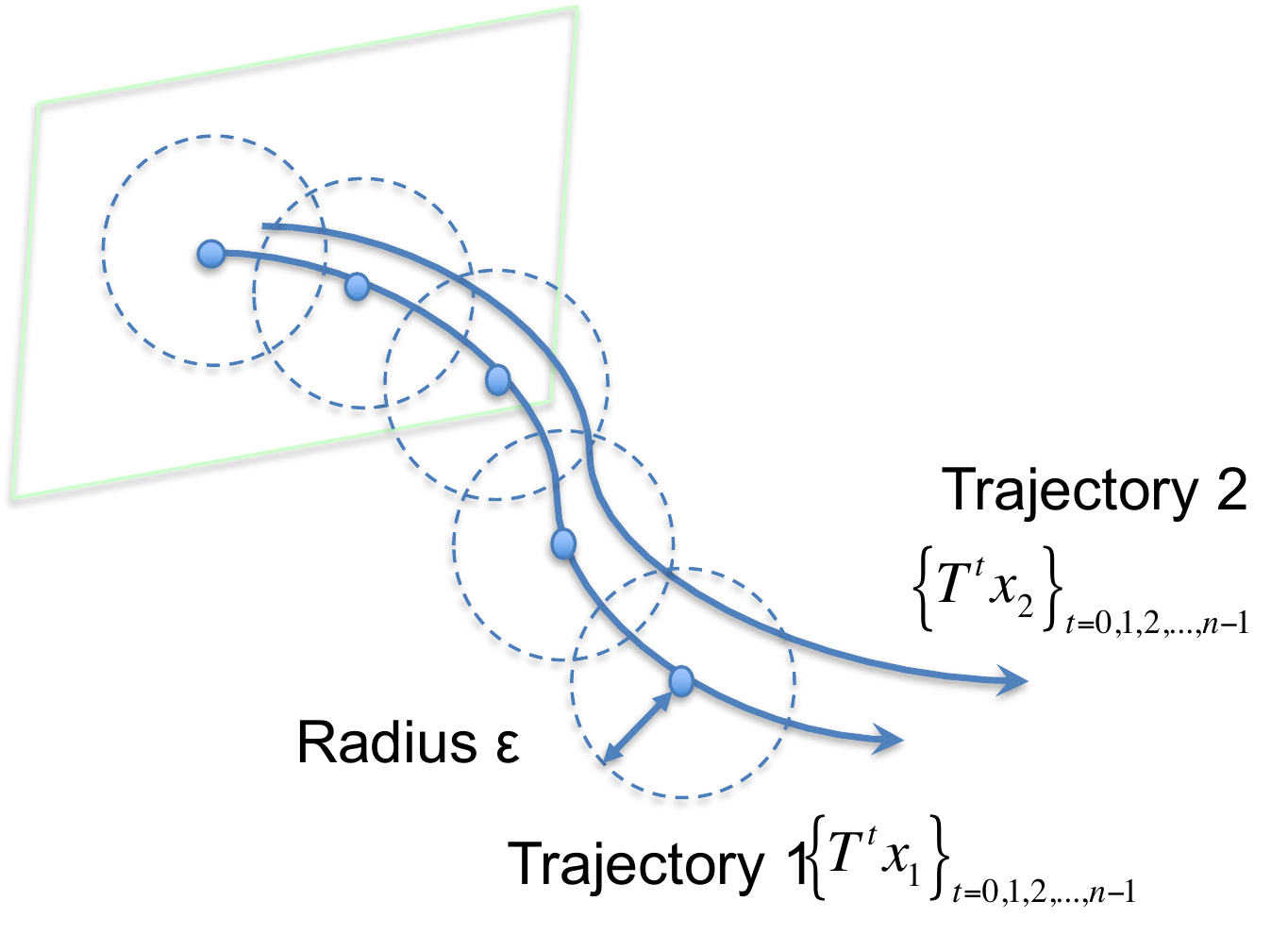}
  \caption{An illustration of two close tracks.}\label{fig:Husheng_capacity_spnnaing}
\end{figure}

The spanning orbits based topological entropy is then defined as follows, where the subscript $s$ means `spanning'.
\begin{definition}\label{def:entropy2}
The topological entropy of the dynamical system $(X,T,S)$ is defined as
\begin{eqnarray}
\left\{
\begin{array}{lll}
H_s(T,n,\epsilon)&=\log_2 r(n,\epsilon)\\
h_{s}(T,\epsilon)&=\lim_{n\rightarrow \infty}\frac{1}{n}H_s(T,\epsilon)\\
h_{s}(T)&=\lim_{\epsilon\rightarrow 0}h_s(T,\epsilon)
\end{array}
\right..
\end{eqnarray}
\end{definition}
Here, $H_s(T,n,\epsilon)$ is the minimum number of bits for the source coding of the orbits with error up to $\epsilon$. Finally the topological entropy $h_{s}(T)$ is the asymptotic coding rate when the time span tends to infinity while the tolerable error tends to 0.

\subsection{Separated Orbits based Definition}

 \begin{figure}
  \centering
  \includegraphics[scale=0.5]{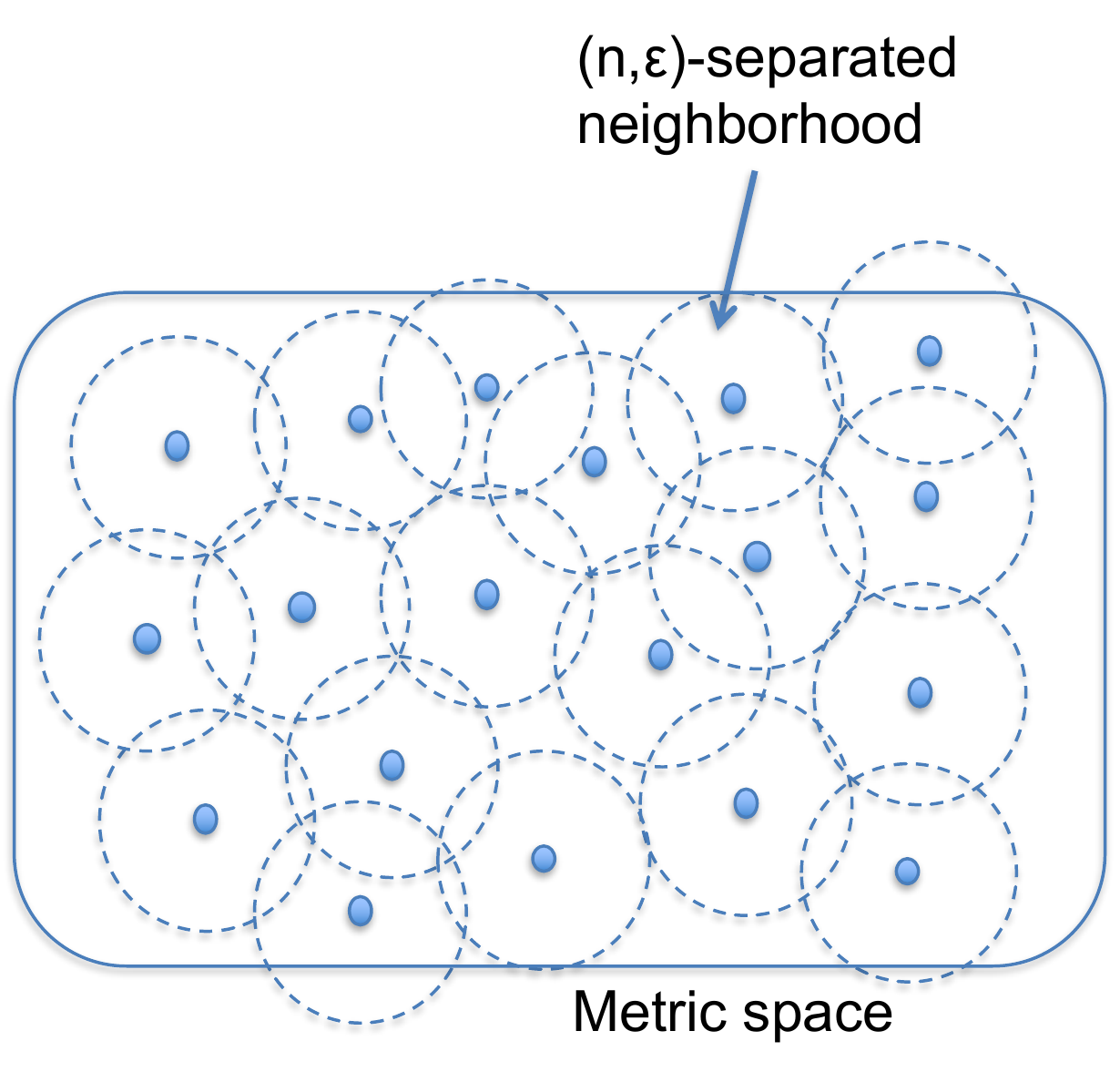}
  \caption{An illustration of an $(n,\epsilon)$-separated set. Note that the circle representing the neighborhood is just an illustration, not a real geometric image.}\label{fig:Husheng_capacity_separated}
\end{figure}

The topological entropy defined from separated orbits is similar to that defined from the spanning orbits. In both cases, we find a set of `typical' initial points or the corresponding orbits, in order to encode the possible orbits. In the spanning orbits case, we begin from a lot of typical orbits, which may be more than enough, and then find the minimum set of orbits. In a contrast, in the separated orbits based definition, we consider points with neighborhoods of given radius. Then, we squeeze as many such orbits (or initial points) as possible into the space. As will be seen later, both selections of typical orbits will yield the same topological entropy in the asymptotic sense. 

To this end, we say that a subset $F\subset X$ is $(n,\epsilon)$-separated, if $d^n(x_i,x_j)>\epsilon$ for all $x_i, x_j\in F$ and $x_i\neq x_j$. Intuitively speaking, the elements in $F$ are separated with distances of at least $\epsilon$. Since $X$ is assumed to be compact, all the possible $(n,\epsilon)$-separated sets have finite elements, as illustrated in Fig. \ref{fig:Husheng_capacity_separated}. We denote the maximal cardinality of all $(n,\epsilon)$-separated sets by
\begin{eqnarray}
s(n,\epsilon)=\max\{|F|: F \mbox{ is $(n,\epsilon)$-separated}\},
\end{eqnarray}
which means the maximum number of points that one can `squeeze' into the space $X$, given the required lower bound of distance $\epsilon$.

Then, we define the metric based topological entropy as
\begin{definition}\label{def:entropy1}
We define the topological entropy as
\begin{eqnarray}
\left\{
\begin{array}{lll}
H_m(n,\epsilon)&=\log_2 s(n,\epsilon)\\
h_m(T,\epsilon)&=\limsup_{n\rightarrow\infty} \frac{1}{n}H_m(n,\epsilon)\\
h_m(T)&=\lim_{\epsilon\rightarrow 0}h_m(T,\epsilon),
\end{array}
\right..
\end{eqnarray}
\end{definition}
Here $H_m(n,\epsilon)$ is the number of bits needed to describe the maximal $(n,\epsilon)$-separated set. Then, the topological entropy $h_m(T)$ is the coding rate (or equivalently, the exponent of increasing orbits) in the asymptotic case (the time span tends to infinity and the tolerable error tends to zero).

\subsubsection{Cover Based Definition} As the third type of definition of topological entropy, the cover based definition of topological entropy does not need the definition of metric. Hence it is still valid in spaces without a metric structure. The requirement is the topological structure, namely the definition of open sets.
A cover of $X$ is defined as a family of open sets that cover $X$. Consider a cover $\mathcal{U}$. 
Due to the assumption of the compactness of $X$, we can always find a finite subset of open sets in $\mathcal{U}$ to cover $X$. For a transformation $T$ and any integer, $T^{-n}(\mathcal{U})$ is also a cover of $X$. A cover $\mathcal{V}$ is a subcover of cover $\mathcal{U}$ if $\mathcal{V}\subset\mathcal{U}$, namely the open sets in $\mathcal{V}$ are also in $\mathcal{U}$ while $\mathcal{V}$ also covers $X$. The minimum cardinality of a subcover of $\mathcal{U}$ is denoted by $N(\mathcal{U})$.

We define the joint of two covers $\mathcal{U}$ and $\mathcal{V}$ as 
\begin{eqnarray}
\mathcal{U}\vee \mathcal{V}=\left\{U\cap V: U\in \mathcal{U},V\in \mathcal{V}\right\}.
\end{eqnarray} 
We further define
\begin{eqnarray}
\mathcal{U}^n=\vee_{n=0,...,n-1}T^{-n}(\mathcal{U}),
\end{eqnarray} 
namely the joint set of the covers obtained from tracing back the open sets in $\mathcal{U}$ along the transformation $T$.

Based on the above definitions on cover sets, we can now define the topological entropy:
\begin{definition}\label{def:entropy3}
The topological entropy of the dynamical system $(X,T,S)$ is defined as 
\begin{eqnarray}
\left\{
\begin{array}{lll}
H_c(\mathcal{U})&=\log_2 N(\mathcal{U})\\
h_c(T,\mathcal{U})&=\lim_{n}\frac{1}{n}H_c\left(\mathcal{U}^n)\right)\\
h_c(T)&=\lim_{\mathcal{U}}\uparrow h_c(T,\mathcal{U})
\end{array}
\right..
\end{eqnarray}
\end{definition}
Here $H_c(\mathcal{U})$ is the number of bits needed to describe the open sets that cover $X$. $h_c(T,\mathcal{U})$ is the coding rate for describing the joint sets generated by tracing back the open sets in $\mathcal{U}$. 

\subsection{Equivalence}
The above three definitions of topological entropy look different, focusing on different aspects of the dynamical system. In particular, the cover based definition is significantly different from the first two. A surprising result is that they are actually equivalent in metric spaces:
\begin{theorem}[\cite{Downarowicz2011}]
In metric spaces, we have 
\begin{eqnarray}
h_s(T)=h_m(T)=h_c(T).
\end{eqnarray}
\end{theorem}
The proof is highly nontrivial. The details can be found in Chapter 6 of \cite{Downarowicz2011}. Due to their equivalence, we use only the notation $h_s(T)$ in the subsequent discussion.

\section{Proof of Theorem \ref{thm:local_Lya} (for Hyperbolic Tangent Case)}\label{appdx:DNN_Lya} 
Here we consider the case of hyperbolic tangent function for the DNN. Before we provide the rigorous proof, we provide the intuition for the possibility of the chaotic behavior in the DNN dynamics. The maximum Lyapunov exponent $\lambda_1$ is determined by the expansion rate of the tangent vector $\mathbf{u}_t=\mathbf{J}_t(\mathbf{x}_0)\mathbf{u}_0$. Only when $\|\mathbf{J}_t(\mathbf{x}_0)\mathbf{u}_0\|$ increases exponentially, is $\lambda_1$ positive. We notice that the matrix $\mathbf{S}_i(\mathbf{x})$ always shrinks $\|\mathbf{J}_t(\mathbf{x}_0)\mathbf{u}_0\|$, since all the diagonal elements are positive but less than 1. Hence, to make $\lambda_1$ positive, the determinant of $\mathbf{W}_i$ should be larger than 1, namely enlarging the vector norm, in an average sense. However, $\mathbf{W}_i$ operates on both $\mathbf{u}_t$ and $\mathbf{x}_t$. Hence, if the determinant of $\mathbf{W}_i$ is too large, it will make the elements in $\mathbf{x}_t$ large, thus making the elements $\mathbf{S}_i(\mathbf{x})$ small and shrinking $\|\mathbf{u}_t\|$. Therefore, there is a tradeoff for the determinant of $\mathbf{W}_i$, to balance the enlarging effect of $\mathbf{W}_i$ and the shrinking effect of $\mathbf{S}_i$. 

For scalar dynamics, namely $d=1$, we carry out numerical calculations to evaluate the Lyapunov exponents. We first tested the scalar dynamical system $x_{t+1}=ax_t+b$. In Fig. \ref{fig:scalar1}, we show the Lyapunov exponent with respect to various values of $a$ and $b$, by fixing the initial value $x_0=0.2$. We observe that the Lyapunov exponent $\lambda_1$ is negative. The value of $b$ has much less impact on the Lyapunov exponent. In Fig. \ref{fig:scalar2}, we show the Lyapunov exponent with respect to various values of $x_0$ and $b$, by fixing $a=3.12$. We observe that, for all the nonzero values of $x_0\in [-1,1]$ and $b$, the Lyapunov exponents are negative. Note that the Lyapunov exponent is different at different initial values, which implies that the scalar dynamical system is not ergodic.

\begin{figure}
  \centering
  \includegraphics[scale=0.4]{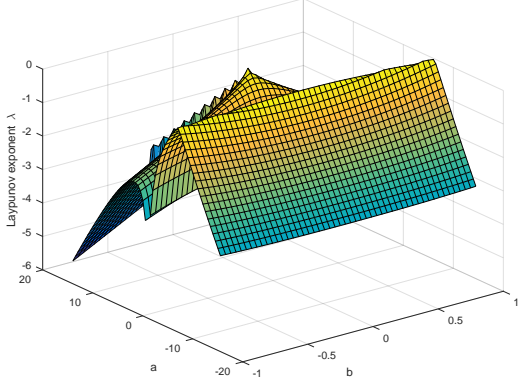}
  \caption{Lyapunov exponent for scalar deterministic dynamics with various $a$ and $b$ given $x(0)=0.2$.}\label{fig:scalar1}
\end{figure}

\begin{figure}
  \centering
  \includegraphics[scale=0.4]{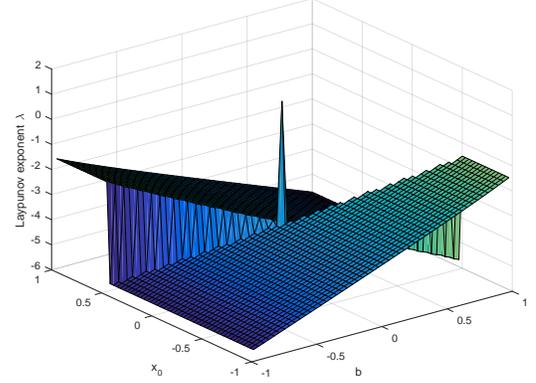}
  \caption{Lyapunov exponent for scalar deterministic dynamics with various $x_0$ and $b$ given $a=3.12$.}\label{fig:scalar2}
\end{figure}

The reason for the lack of chaotic behavior in the scalar dynamics of DNN is because $\mathbf{u}_t$ and $\mathbf{x}_t$ are aligned in the same direction in $\mathbb{R}$. Therefore, the enlarging effect of $\mathbf{W}_i$ will backfire at $\mathbf{u}_t$ via $\mathbf{S}_t$ directly. However, in higher dimensional spaces, it is possible to tune the parameters $\mathbf{W}_i$ such that $\|\mathbf{x}_t\|$ is prevented from being too large while $\|\mathbf{u}_t\|$ will be enlarged exponentially, since $\mathbf{x}_t$ and $\mathbf{u}_t$ may not be aligned in the same subspace. 

The following proof of Theorem \ref{thm:local_Lya} follows the above idea of decoupling the system state $\mathbf{x}_t$ and tangent vector $\mathbf{u}_t$:

 \begin{proof}
 There are two approaches to generate chaos in the DNN dynamics by designing the DNN parameters carefully. We assume $d=2$, namely the dynamics of DNN is in $\mathbb{R}^2$. For higher dimensional dynamical systems, we can project the system state to 2-dimensional subspaces, thus not affecting the conclusion. 
 
 \begin{figure}
  \centering
  \includegraphics[scale=0.4]{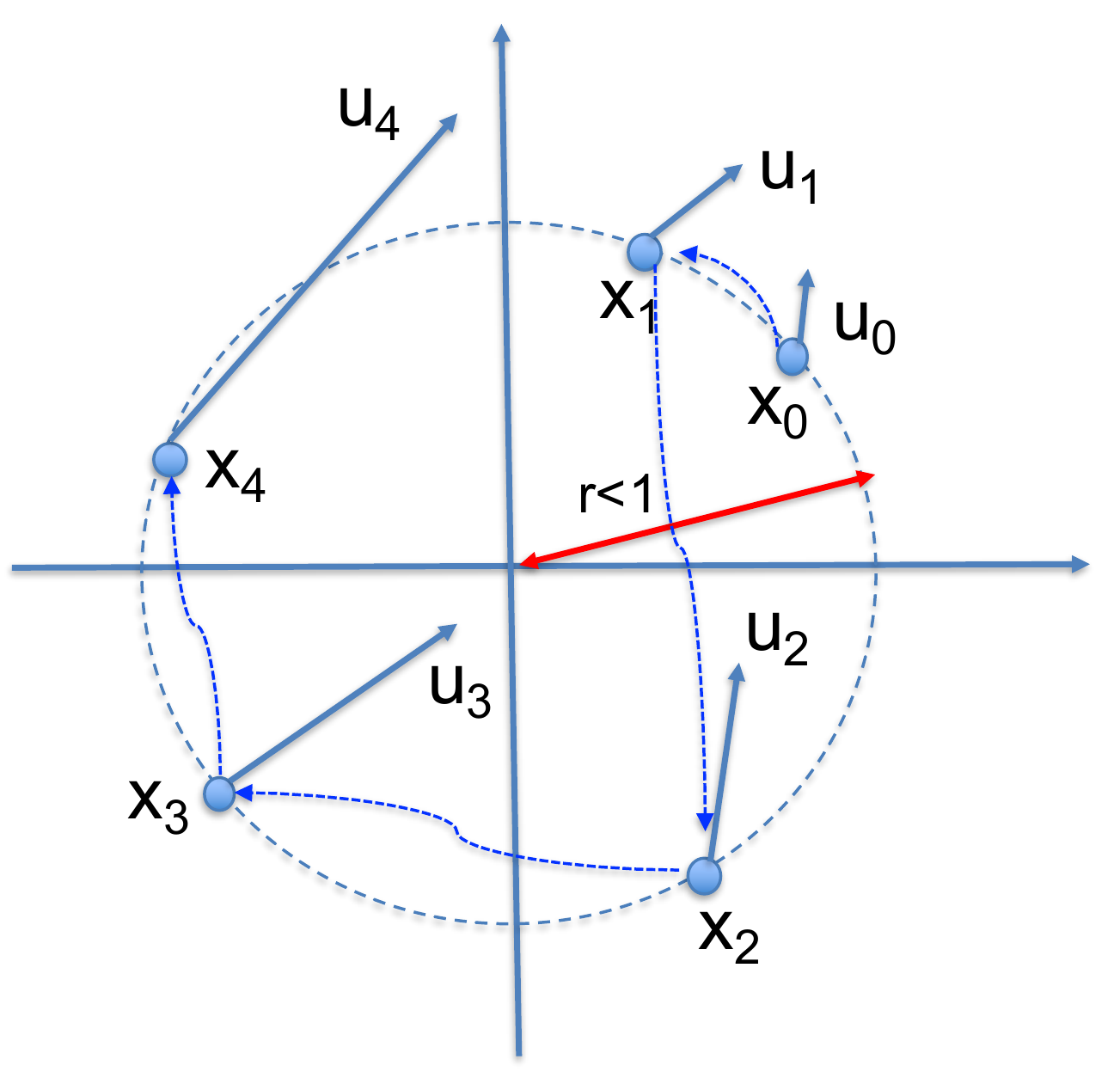}
  \caption{Design of linear mappings for chaos in DNN.}\label{fig:chaotic}
\end{figure}

\textbf{Approach 1}: In this first approach, we force the system state $\mathbf{x}_t$ to have a constant norm, as illustrated in Fig. \ref{fig:chaotic}. Denote by $\mathbf{u}_0\neq 0$ the initial direction and by $\mathbf{x}_0=(x_{01},x_{02})$ the initial point of the successive mapping. We assume $r=\|\mathbf{x}_0\|< 1$. We further assume that $\mathbf{u}_0$ and $\mathbf{x}_0$ are linearly independent. We also fix a positive number $A>1$.

For the $t$-th layer, the output is $\mathbf{x}_t=\tanh(\tilde{\mathbf{x}}_t)$, where $\tilde{\mathbf{x}}_t=\mathbf{W}_t\mathbf{x}_{t-1}$. The mapping maps the derivative $\mathbf{u}_{t-1}$ to $\mathbf{u}_t$ in (\ref{eq:dyn_u}). We assume that $\mathbf{u}_{t-1}$ and $\mathbf{x}_{t-1}$ are linearly independent.
We set the linear mapping $\mathbf{W}_t$ to be a matrix with real eigenvector $\mathbf{u}_{t-1}$ corresponding to eigenvalue $A$. We further assume that $\mathbf{W}_t$ satisfies
\begin{eqnarray}
\mathbf{W}_t\mathbf{x}_{t-1}=\arg\tanh(\mathbf{x}_t),
\end{eqnarray}
such that $|\mathbf{x}_t|=r<1$. Due to the constraints, the matrix $\mathbf{W}_t$ can be obtained from the following equation:
\begin{eqnarray}\label{eq:design}
\left\{
\begin{array}{ll}
&\mathbf{W}_t\mathbf{u}_{t-1}=A\mathbf{u}_{t-1}\\
&\mathbf{W}_t\mathbf{x}_{t-1}=\tilde{\mathbf{x}}_t
\end{array}
\right.,
\end{eqnarray}
which can be rewritten as
\begin{eqnarray}
\mathbf{W}_t\mathbf{B}_t=\mathbf{C}_t,
\end{eqnarray}
where $\mathbf{B}_t=(\mathbf{u}_{t-1},\mathbf{x}_{t-1})$ and $\mathbf{C}_t=(A\mathbf{u}_{t-1},\tilde{\mathbf{x}}_{t})$. Due to the assumption that $\mathbf{u}_{t-1}$ and $\mathbf{x}_{t-1}$ are linearly independent, the matrix $\mathbf{B}_t$ is of full rank and is thus invertible. Therefore, we can obtain $\mathbf{W}_t$ by 
\begin{eqnarray}
\mathbf{W}_t=\mathbf{C}_t\mathbf{B}_t^{-1}.
\end{eqnarray}

Then, the derivative after the $t$-th round of mapping is given by
\begin{eqnarray}\label{eq:ut}
\mathbf{u}_t&=&\begin{pmatrix}
1-x_{t1}^2& 0\\
0 & 1-x_{t2}^2
\end{pmatrix}
\mathbf{W}_t\mathbf{u}_{t-1}\nonumber\\
&=&A\begin{pmatrix}
1-x_{t1}^2& 0\\
0 & 1-x_{t2}^2
\end{pmatrix}
\mathbf{u}_{t-1}
\end{eqnarray}

Now, we need to guarantee that $\mathbf{u}_t$ and $\mathbf{x}_t$ are linearly independent for the purpose of induction. Notice that the only constraint on $\mathbf{x}_t$ is $\|\mathbf{x}_t\|=r$. Suppose that for any point $\mathbf{x}_t$ on the circle $\|\mathbf{x}_t\|=r$, we have
\begin{eqnarray}\label{eq:rank0}
u_{t1}x_{t_2}=u_{t2}x_{t_1},
\end{eqnarray}
namely $\det(\mathbf{B}_t)=0$. Then, substituting (\ref{eq:ut}) into (\ref{eq:rank0}), we have
\begin{eqnarray}\label{eq:three_order}
u_{t-1,1}x_{t_2}(1-x_{t1}^2)=u_{t-1,2}x_{t_1}(1-x_{t2}^2).
\end{eqnarray}

Unless both $u_{t-1,1}=u_{t-1,2}=0$, which is impossible due to the assumption $\mathbf{u}_0\neq 0$ and the induction in (\ref{eq:ut}), the equation in (\ref{eq:three_order}) defines a curve of degree 3. According to the Bezout theorem \cite{Shafarevich2013}, there are at most six real intersection points for the curve in (\ref{eq:three_order}) and the circle $\|\mathbf{x}_t\|=r$. Therefore, one can always find a point on $\|\mathbf{x}_t\|=r$ such that (\ref{eq:three_order}) does not hold, namely $\mathbf{u}_t$ and $\mathbf{x}_t$ are linearly independent.

Using induction on (\ref{eq:ut}), we have
\begin{eqnarray}\label{eq:evol_u}
\mathbf{u}_t=A^{t}\prod_{l=1}^t\begin{pmatrix}
1-x_{l1}^2& 0\\
0 & 1-x_{l2}^2
\end{pmatrix}\mathbf{u}_0.
\end{eqnarray}
Hence, we have
\begin{eqnarray}\label{eq:explosion}
\frac{\|\mathbf{u}_t\|}{\|\mathbf{u}_0\|}\geq \left(A(1-r^2)\right)^t.
\end{eqnarray}

When we choose a sufficiently large $A$ such that the right hand side of (\ref{eq:explosion}) is greater than 1, $\frac{|\mathbf{u}_t|}{|\mathbf{u}_0|}$ increases exponentially, thus making the Laypunov exponent positive. Moreover, the Laypunov exponent can be made arbitrarily large due to the arbitrary selection of $A$.
 
\textbf{Approach 2}: In the second approach, the DNN dynamics force the system state to rational points in $\mathbb{R}^2$. To this end, we consider a compact set $\Omega\in [-1,1]^2$ with nonzero Lebeasgue measure. We assume
\begin{eqnarray}
\max_{\mathbf{x}\in \Omega}\|\mathbf{x}\|_1\leq r<1,
\end{eqnarray}
where the maximum is achieved due to the compactness of $\Omega$.

We can enumerate all the rational points (i.e., the coordinates are rational numbers) in $\Omega$, except for the points on the lines $x_1=x_2$ and $x_1=-x_2$, in the sequence $\mathbf{X}=\left\{\mathbf{x}_i\right\}_{i=1,2,3,...}$. It is easy to verify that the set $\mathbf{X}$ is dense in $\Omega$. We also choose the initial vector $\mathbf{u}_0$ such that
\begin{eqnarray}
\frac{u_{01}}{u_{02}}\notin \mathbb{Q},
\end{eqnarray}
namely the ratio of the two coordinates in $\mathbf{u}_0$ is an irrational number.

Then, we design the linear mappings in the same way as in Approach 1, except that the design of $\mathbf{W}_t$ in (\ref{eq:design}) is changed to
\begin{eqnarray}\label{eq:design2}
\left\{
\begin{array}{ll}
&\mathbf{W}_t\mathbf{u}_{t-1}=A\mathbf{u}_{t-1}\\
&\mathbf{W}_t\mathbf{x}_{t-1}=\arg\tanh\left(\mathbf{x}_t\right)
\end{array}
\right.,
\end{eqnarray}
which is $\mathbf{W}_t\mathbf{B}_t=\mathbf{C}_t$. Here the major goal is to assure the reversibility of the matrix $\mathbf{B}_t$, namely $\mathbf{u}_t$ and $\mathbf{x}_t$ are linearly independent. To this end, we assume that $\mathbf{u}_t$ and $\mathbf{x}_t$ are linearly dependent, which implies
\begin{eqnarray}\label{eq:linear}
\frac{u_{t1}}{u_{t2}}=\frac{x_{t1}}{x_{t2}}.
\end{eqnarray}

According to (\ref{eq:evol_u}), we have
\begin{eqnarray}
\frac{u_{t1}}{u_{t2}}=\frac{u_{01}}{u_{02}}\prod_{l=1}^t\left(\frac{1-x_{l1}^2}{1-x_{l2}^2}\right),
\end{eqnarray}
which is irrational since $\frac{u_{01}}{u_{02}}$ is irrational and $\prod_{l=1}^t\left(\frac{1-x_{l1}^2}{1-x_{l2}^2}\right)$ is rational (recall that both $x_{l1}$ and $x_{l2}$ are rational). However, $\frac{x_{t1}}{x_{t2}}$ is rational. Hence, (\ref{eq:linear}) does not hold, since the left hand side is irrational while the right hand side is rational. Therefore, $\mathbf{u}_t$ and $\mathbf{x}_t$ are linearly independent, which assures the existence of solution in (\ref{eq:design2}). This concludes the proof.

 \end{proof}
 
 \section{Proof of Theorem \ref{thm:local_Lya2} (for the ReLU Case)}\label{appdx:DNN_Lya2}
 
 \begin{proof}
 Similarly to the hyperbolic tangent case, we consider only the 2-dimensional case. Since the angle is given by $\theta(t)=\arctan \left(\frac{x_2(t)}{x_1(t)}\right)$, we have
 \begin{small} 
\begin{eqnarray}
\frac{\partial \rho(\mathbf{x}(T))}{x_1(0)}&=&\frac{1}{1+\left(\frac{x_2(T)}{x_1(T)}\right)^2}\nonumber\\
&\times&\left(-\frac{x_2(T)}{x^2_1(T)}\frac{\partial x_1(T)}{\partial x_1(0)}+\frac{1}{x_1(T)}\frac{\partial x_2(T)}{\partial x_1(0)}\right),
\end{eqnarray} 
\end{small}
where according to the chain's rule the partial derivative is the sum of the chains of partial derivatives along all possible paths from $x_{1}(0)$ to $x_1(T)$ or $x_2(T)$, namely 
\begin{eqnarray}
\frac{\partial x_1(T)}{\partial x_1(0)}=\sum_{i_1,...,i_{T-1}=1}^2\frac{\partial x_1(T)}{\partial x_{i_{T-1}}(T-1)}\times \cdots\times\frac{\partial x_{i_1}(1)}{\partial x_1(0)},
\end{eqnarray}
and
\begin{eqnarray}
\frac{\partial x_2(T)}{\partial x_1(0)}=\sum_{i_1,...,i_{T-1}=1}^2\frac{\partial x_2(T)}{\partial x_{i_{T-1}}(T-1)}\times \cdots\times\frac{\partial x_{i_1}(1)}{\partial x_1(0)},
\end{eqnarray}
and
\begin{eqnarray}
&&\frac{\partial x_i(t)}{\partial x_j(t-1)}\nonumber\\
&=&\left\{
\begin{array}{ll}
W_{ij}^{t-1}, \mbox{ if }W_{i1}^{t-1}x_1(t-1)+W_{i2}^{t-1}x_2(t-1)>0\\
0,\mbox{ otherwise}
\end{array}
\right..
\end{eqnarray}

Now, we set the parameters of DNN in the following manner:
\begin{itemize}
\item $W_{i1}^{t-1}x_1(t-1)+W_{i2}^{t-1}x_2(t-1)<0$ for $t=1,t_0,2t_0,...$ (or equivalently $x_2(t)=0$), where $t_0$ is a positive integer.
\item $W_{21}^t=0$ and $|W_{22}|<C$.
\item $W_{11}^t=a>1$ and $b_1=(1-a)x_0$.
\end{itemize}

Given the above parameters, we have
\begin{eqnarray}
\frac{\partial x_2(T)}{\partial x_1(0)}=0,\qquad \forall T>1,
\end{eqnarray}
and 
\begin{eqnarray}
\frac{\partial x_1(T)}{\partial x_1(0)}=a^T,
\end{eqnarray}
and $x_1(T)=x_0$ and $x_2(T)$ is upper bounded by a certain constant. Hence, it is easy to verify that 
\begin{eqnarray}
\frac{\partial \rho(\mathbf{x}(T))}{x_1(0)}=O(a^T),
\end{eqnarray} 
which tends to infinity since $a>1$ and implies $\left\|\nabla_{\mathbf{x}(0)}\rho(\mathbf{x}(T))\right\|\rightarrow\infty$. This concludes the proof.

 \end{proof} 
 
\section{Proof of Theorem \ref{thm:global} and Justification}\label{appx:global}
We first provide numerical results to justify the assumption in Theorem \ref{thm:global}, namely $\sum_{i=1}^d x_{ti}^2$ is a constant $C$ when both $d$ and $t$ are sufficiently large. Figures \ref{fig:norm1} and \ref{fig:norm2} show the time evolution and variance of the norm normalized by $\sqrt{d}$, respectively. We observe that the normalized norm becomes more concentrated as the dimension becomes larger. Note that a sufficiently large $t$ means that the distribution of $\mathbf{x}_t$ becomes stationary, while a large $d$ implies the concentration of measure for $\mathbf{x}_t$, thus resulting in the constant norm of $\mathbf{x}_t$.

\begin{figure}
  \centering
  \includegraphics[scale=0.4]{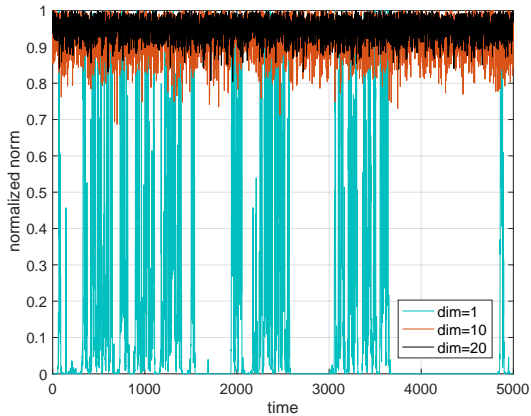}
  \caption{Time evolution for the normalized norm of system state with different dimensions, where dim means the dimension $d$.}\label{fig:norm1}
\end{figure}

\begin{figure}
  \centering
  \includegraphics[scale=0.4]{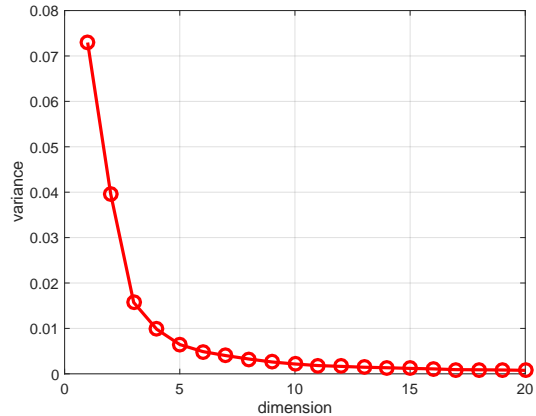}
  \caption{Variance of normalized norm of system state with different dimensions.}\label{fig:norm2}
\end{figure}

To prove Theorem \ref{thm:global}, we need the following simple lemma:
\begin{lemma}\label{lem:exist}
There exist a positive solution to the following equation with the unknown $R$:
\begin{eqnarray}\label{eq:ind_fix}
R=\sqrt{E[\tanh^2(z)]},
\end{eqnarray}
where $z$ is a Gaussian random variable with zero expectation and variance $\sigma^2 R$, when $\sigma^2>1$.
\end{lemma}
\begin{proof}
It is easy to verify that the righthand side of (\ref{eq:ind_fix}) is a continuous function of $R$.
Then, we let $R$ be sufficiently small, such that $\tanh(z)$ is very close to a linear function (since $\tanh$ is close to a linear function when the argument is small), namely
\begin{eqnarray}
E[\tanh^2(z)]=\sigma^2R+o(R),
\end{eqnarray}
which is larger than the left hand side of (\ref{eq:ind_fix}).

We let $R$ tend to infinity. The right hand side of (\ref{eq:ind_fix}) is bounded by 1 while the left hand side is larger than the right band side. Due to the continuity of both sides of (\ref{eq:ind_fix}), there exists at least one positive solution.
\end{proof}

We also need the following lemma for the proof of Theorem \ref{thm:global}.
\begin{lemma}\label{lem:scalar_case}
We consider the following scalar dynamics
\begin{eqnarray}\label{eq:scalar_dyn}
x(t+1)=\tanh(\alpha x(t)),
\end{eqnarray}
where $\alpha$ is a Gaussian random variable with zero expectation and variance $\sigma^2$.
For a fixed $\alpha>0$, we consider the equation 
\begin{eqnarray}
x=\tanh(\alpha x),
\end{eqnarray}
whose nonzero solution is denoted by $x=h(\alpha)$. Then, we have
\begin{eqnarray}
Var[x]\leq 4\tanh \left(\sigma\sqrt{\frac{2}{\pi}}h\left(\sigma\sqrt{\frac{2}{\pi}}\right)\right).
\end{eqnarray}
\end{lemma}

\begin{proof}
The variance of $x$ in the stationary stage is given by
\begin{eqnarray}\label{eq:var}
Var(x)=\int \tanh^2(uv)f_x(u)f_\alpha(v)dudv,
\end{eqnarray}
where $f_x$ is the stationary distribution of $x$ and $f_\alpha$ is the pdf of Gaussian distribution. It is easy to verify 
\begin{eqnarray}\label{eq:var}
Var(x)=4\int_0^\infty\int_0^\infty \tanh^2(uv)f_x(u)f_\alpha(v)dudv,
\end{eqnarray}
namely we can consider only positive $x$ and $\alpha$ due to the symmetry.

We fix $\alpha$. Then, we have
\begin{eqnarray}
\int_0^\infty \tanh^2(\alpha x)dx\leq \tanh^2 (\alpha E^+[x]),
\end{eqnarray}
where $E^+[x]$ is the conditional expectation of $x$ under $2f_x$ in $[0,\infty)$, due to the Jensen's inequality and the concavity of $\tanh(x)$ when $x>0$. When we take the expectation over the randomness of $\alpha$, we have
\begin{eqnarray}\label{eq:var3}
E_{\alpha}^+\left[\tanh^2 (\alpha E[x])\right]\leq \tanh^2(E^+[\alpha]E^+[x]),
\end{eqnarray}
which is again because of the Jensen's inequality.

Moreover, we have, for the stationary distribution,
\begin{eqnarray}
E^+[x]&=&E_{\alpha,x}^+[\tanh(\alpha x)]\nonumber\\
&\leq &E_{\alpha}^+[\tanh(\alpha E^+[x])]\nonumber\\
&\leq &\tanh(E^+[\alpha]E^+[x]).
\end{eqnarray}
where both $E^+[x]$ and $E^+[\alpha]$ are averaged over $[0,\infty)$. 

Then, we have
\begin{eqnarray}\label{eq:leq}
E^+[x] \leq h(E^+[\alpha]).
\end{eqnarray}
This inequality can be obtained by proving the concavity of the function $\int_0^\infty \tanh(\alpha x)f(\alpha)d\alpha$ when $x>0$, where the argument is $x$, which stems from the concavity of $\tanh(x)$ when $x>0$, as illustrated in Fig. \ref{fig:intersection}. 

\begin{figure}
  \centering
  \includegraphics[scale=0.5]{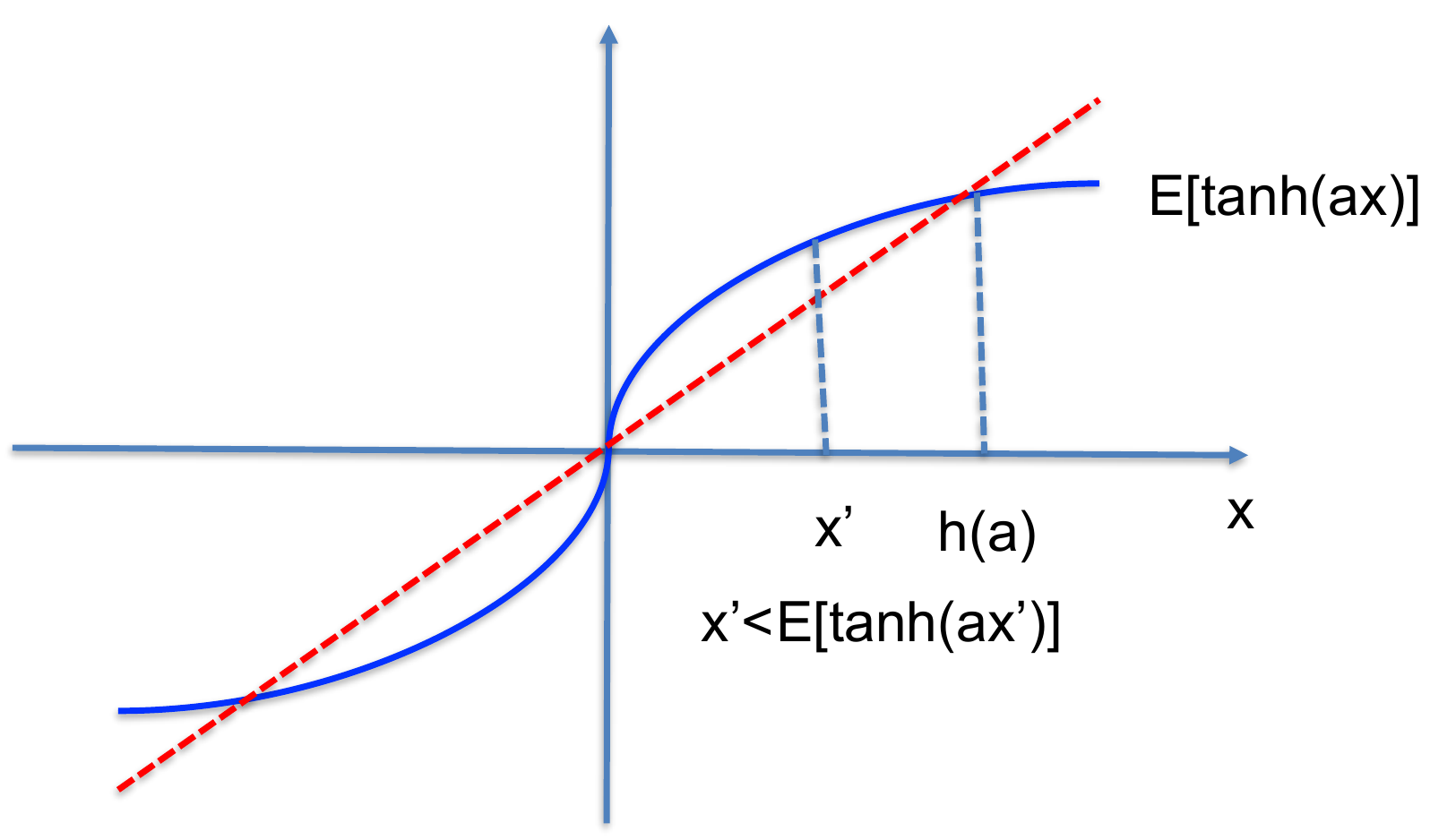}
  \caption{An illustration of the inequality in (\ref{eq:leq})}\label{fig:intersection}
\end{figure}

Therefore, combining (\ref{eq:var}), (\ref{eq:var3}) and (\ref{eq:leq}), we have 
\begin{eqnarray}
Var(x)\leq 4\tanh(E[\alpha]h(E[\alpha])),
\end{eqnarray}
where the variance of $x$ is over $(-\infty,\infty)$ while the expectation of $\alpha$ is over $[0,\infty)$.

It is easy to verify
\begin{eqnarray}
E^+[\alpha|\alpha\geq 0]=\sigma\sqrt{\frac{2}{\pi}},
\end{eqnarray}
which is the average of Gaussian random variable in the positive range. This concludes the proof.

\end{proof}

Then, we provide the full proof for Theorem \ref{thm:global}:

\begin{proof}
First we need to bound the constant $C$. To this end,
we denote by $\phi$ the direction of $\mathbf{x}$ and denote by $r$ its amplitude $\|\mathbf{x}\|$. Then, we solve the equation with unknown $r$:
\begin{eqnarray}\label{eq:fixed}
r=E\left[\left\|\tanh (\mathbf{W}\mathbf{x}_{\phi}(r))\right\|\right].
\end{eqnarray}

For each direction $\phi$, we have a fixed point $r(\phi)$ in (\ref{eq:fixed}), using a proof similar to that of Lemma \ref{lem:exist}. 

Then, we consider the following optimization problem:
\begin{eqnarray}\label{eq:max_phi}
\max_{\phi}  E\left[\left\|\tanh (\mathbf{W}\mathbf{x}_{\phi}(r(\phi)))\right\|\right],
\end{eqnarray}
where the expectation is over the randomness of $\mathbf{W}$ and $r(\phi)$ is the solution to (\ref{eq:fixed}) given the direction $\phi$.

Due to the assumption of Gaussian matrix for $\mathbf{W}$, $\mathbf{W}\mathbf{x}$ is a Gaussian distributed vector with zero expectation and covariance matrix given by
\begin{eqnarray}
\mathbf{\Sigma}=E\left[\mathbf{W}\mathbf{x}\mathbf{x}^T\mathbf{W}^T\right]
\end{eqnarray}
we have
\begin{eqnarray}
\mathbf{\Sigma}_{ij}&=&E\left[\sum_{k,m}\mathbf{W}_{ik}\mathbf{x}_{k}\mathbf{x}_{m}\mathbf{W}_{jm}\right]\nonumber\\
&=&\left\{
\begin{array}{ll}
\frac{\sigma^2\|\mathbf{x}\|^2}{d},\qquad \mbox{if }i=j\\
0,\qquad \mbox{if }i\neq j
\end{array}
\right..
\end{eqnarray}
Hence, we have
\begin{eqnarray}
\mathbf{\Sigma}&=&\frac{\sigma^2 \|\mathbf{x}\|^2}{d}\mathbf{I}\nonumber\\
&=&\frac{\sigma^2 r^2}{d}\mathbf{I},
\end{eqnarray}
which is a diagonal matrix.

Therefore, the maximization in (\ref{eq:max_phi}) is actually independent of $\phi$. We have 
\begin{eqnarray}\label{eq:fixed_eq}
r=E\left[\left\|\tanh(\mathbf{z})\right\|\right],
\end{eqnarray}
where $\mathbf{z}\sim \mathcal{N}(\mathbf{0},\mathbf{\Sigma})$.

Since $\mathbf{z}$ has i.i.d. components, (\ref{eq:fixed_eq}) can be rewritten as
\begin{eqnarray}
R=\frac{1}{\sqrt{2\pi R}}\int_\infty^{-\infty} \tanh^2 (\alpha x) e^{-\frac{x^2}{2R}}p_n(\alpha)dxd\alpha,
\end{eqnarray}
where $R=\frac{r^2}{d}$ and $\alpha$ is a Gaussian random variable with zero expectation, $p_n$ is the normal distribution, and variance $\sigma^2$. It is easy to see that $R$ is simply the variance of $x$ in the scalar dynamics of (\ref{eq:scalar_dyn}) in Lemma \ref{lem:scalar_case}. By applying the conclusion in Lemma \ref{lem:scalar_case}, we have
\begin{eqnarray}\label{eq:Rles}
R\leq 4\tanh \left(\sigma\sqrt{\frac{2}{\pi}}h\left(\sigma\sqrt{\frac{2}{\pi}}\right)\right).
\end{eqnarray}

Then, for each stage of the dynamics, we have
\begin{eqnarray}\label{eq:beta_t}
\beta_t&=&(1-R)\frac{E\left[\|\mathbf{W}_t\mathbf{u}_{t-1}\|^2\right]}{E\left[\|\mathbf{u}_{t-1}\|^2\right]}\nonumber\\
&=&(1-R)\frac{d\sigma^2E\left[\|\mathbf{u}_{t-1}\|^2\right]}{E\left[\|\mathbf{u}_{t-1}\|^2\right]}\nonumber\\
&=&(1-R)d\sigma^2.
\end{eqnarray}

The conclusion is achieved by applying (\ref{eq:Rles}) to (\ref{eq:beta_t}) and Assumption \ref{assum:mean}.
\end{proof}

We can also find an expression for $h$. We apply the Lagrange inversion formula \cite{Bruijn2010} by consider the following function
\begin{eqnarray}
f(x,\alpha)=\tanh(\alpha x)-x,\qquad x\in [0,1],
\end{eqnarray}
and then obtain the expression for $h(\alpha)$:
\begin{eqnarray}
h(\alpha)=0.5+\sum_{n=1}^\infty g_n(\alpha)\frac{(-f(0.5,\alpha))^n}{n!},
\end{eqnarray}
where 
\begin{eqnarray}
g_n(\alpha)=\lim_{w\rightarrow 0.5}\left[\frac{d^{n-1}}{dw^{n-1}}\left(\frac{w-0.5}{f(w,\alpha)-f(0.5,\alpha)}\right)^n\right].
\end{eqnarray}

\section{Classification Complexity with Other Related Metrics}\label{appdx:related}
We discuss the qualitative relationship (including conjectures) of the classification complexity in Definition \ref{def:complexity0} to other related metrics describing system complexities. 

\subsection{Hausdorff Dimension}
We notice that the partition based complexity in Definition \ref{def:complexity0} is similar to complexities describing how chaotic the system is. One of such metrics is the Hausdorff dimension \cite{Falconer2003} in metric space, which is defined as follows:
\begin{definition}[Hausdorff Dimension]
Consider a metric space $X$ and a subset $S\subset X$. We say that a countable set of balls $\{B_i\}_{i=1,2,...}$ cover $S$ if $S\in \cup_i B_i$. The set of all such ball coverings is denoted by $V(S)$. For integer $d\in [0,\infty)$, the $d$-th dimensional Hausdorff content of $S$ is defined as
\begin{eqnarray}
C_H^d(S)=\inf_{\mathcal{B}\in V(S)}\left\{\sum_{i\in I} r_i^d\bigg| \mbox{$r_i$: radius of ball $i$ in $\mathcal{B}$}\right\}.
\end{eqnarray}
Then, the Hausdorff dimension of $S$ is defined as
\begin{eqnarray}
\dim_H(S)=\inf\{d\geq 0: C_H^d(S)=0\}.
\end{eqnarray}
\end{definition}

We can also use the Hausdorff dimension to measure the complexity of the hybrid Borel sets in Definition \ref{def:complexity0}. We notice the following differences between the setup in the Hausdorff dimension and Definition \ref{def:complexity0}:
\begin{itemize}
\item Definition \ref{def:complexity0} relies on partitions while the Hausdorff dimension depends on coverings. 
\item The Hausdorff dimension depends on the asymptotic case of $C_H^d(S)$ tending to 0 and arbitrary value of radius. In Definition \ref{def:complexity0}, we have a limited granularity $\epsilon>0$ and the content of all hybrid sets is nonzero. 
\end{itemize}

Therefore, we can define an $(\epsilon,\delta)$-Hausdorff dimension $d(\epsilon,\delta)$ as
\begin{eqnarray}
H(\epsilon,\delta,X)=\inf_{P:partition}\inf\left\{d|N(P|X)\epsilon^d\leq \delta\right\}.
\end{eqnarray}

Our conjecture is that, under certain conditions, we have
\begin{eqnarray}
H(\epsilon,\delta)\approx \frac{\log\delta -C(\mathbf{X},\epsilon)}{\log \epsilon}.
\end{eqnarray}
and an asymptotic equality given by
\begin{eqnarray}
&&\lim_{\epsilon,\delta\rightarrow 0,|\mathbf{X}|\rightarrow\infty}H(\epsilon,\delta)\nonumber\\
&=&\lim_{\epsilon,\delta\rightarrow 0,|\mathbf{X}|\rightarrow\infty} \frac{\log\delta -C(\mathbf{X},\epsilon)}{\log \epsilon}.
\end{eqnarray}
The rationale is that both the Hausdorff dimension and the classification complexity will be achieved by very similar partitions. However, we are still unable to prove it.

\subsection{Kolmogorov Complexity}
Another metric to measure the complexity of a system is the Kolmogorov complexity $K$ \cite{Li2008}. Informally speaking, the Kolmogorov complexity of a deterministic sequence $\mathbf{x}=(x_1,...,x_n)$ is the minimum number of bits necessary to describe the sequence $\mathbf{x}$. Although the Kolmogorov complexity is non-computable, it can be approximated by the codeword length of universal source coding, up to a universal constant.

We notice that the hybrid Borel sets essentially describe the classification surface of the two classes. We can consider a partition as a quantization with an upper bound $\epsilon$ for the quantization error. Hence, the number of bits to describe a set in the partition is approximately given by
\begin{eqnarray}
\log_2 \mbox{vol}(\Omega)-d\log_2\epsilon.
\end{eqnarray}
Then, the number of bits to list the hybrid (boundary) sets in the partition is approximately
\begin{eqnarray}
K\approx 2^{C(\mathbf{X},\epsilon)}\left(\log_2 \mbox{vol}(\Omega)-d\log_2\epsilon\right).
\end{eqnarray}
Therefore the classification complexity is approximately
\begin{eqnarray}
C(\mathbf{X},\epsilon)\approx \log_2\frac{K}{\left(\log_2 \mbox{vol}(\Omega)-d\log_2\epsilon\right)}.
\end{eqnarray}

\section{Proof of Theorem \ref{thm:layer_num}}\label{appdx:layer_num}
\begin{proof}
The proof simply relies on the definitions. According to Definition \ref{def:complexity0}, regardless of the $\epsilon$-partition $P$ in which each set has a radius no more than $\epsilon$, there are at least $N(P|\mathbf{X})$ hybrid sets in $P$. Each hybrid set contains at least two points $\mathbf{x}_1\in \mathbf{X}_1$ and $\mathbf{x}_2\in \mathbf{X}_2$. Obviously, we have
\begin{eqnarray}
\|\mathbf{x}_1-\mathbf{x}_2\|\leq \epsilon.
\end{eqnarray}

Since the sample set is affine separable with $\epsilon$-margin after $D$ rounds of nonlinear transformations in the DNN, we have 
\begin{eqnarray}
\left\{
\begin{array}{ll}
&\mathbf{w}^T\mathbf{z}_1-b\geq \frac{\epsilon}{2}\\
&\mathbf{w}^T\mathbf{z}_2-b\leq -\frac{\epsilon}{2}
\end{array}
\right..
\end{eqnarray}
where $\mathbf{z}_i=T_D\circ T_{D-1}\circ ... \circ T_1(\mathbf{x}_i)$, $i=1,2$.
Subtracting the second inequality from the second one, we have
\begin{eqnarray}
|\mathbf{w}^T(\mathbf{z}_1-\mathbf{z}_2)|\geq \epsilon.
\end{eqnarray}
Hence, we have (recall that $\|\mathbf{w}^T\|$=1) 
\begin{eqnarray}
&&\|\mathbf{z}_1-\mathbf{z}_2\|\nonumber\\
&=&\|\mathbf{w}\|\|\mathbf{z}_1-\mathbf{z}_2\|\nonumber\\
&\geq& |\mathbf{w}^T(\mathbf{z}_1-\mathbf{z}_2)|\nonumber\\
&\geq& \epsilon.
\end{eqnarray}
Therefore, each pair of $\mathbf{x}_1$ and $\mathbf{x}_2$ in the hybrid set will be separated by a distance of at least $\epsilon$ after $D$ rounds of the nonlinear transformations in the DNN.  

Then, any $(D,\epsilon)$-spanning set of the DNN dynamics in Definition \ref{eq:metric_n} has a cardinality more than $2^{C(\mathbf{X},\epsilon)+1}$, namely
\begin{eqnarray}
2^{H_s(\{T_n\}_{n=1,..,D},D,\epsilon)}\geq 2^{C(\mathbf{X},\epsilon)+1}.
\end{eqnarray}
This concludes the proof.
\end{proof}

\section{Proof of Theorem \ref{thm:vc_bound}}\label{appdx:vc_bound}
\begin{proof}
The proof is similar to that of Theorem \ref{thm:layer_num}, since it is also based on bounding the number of distinct paths in the $D$ rounds of transformations of the DNN dynamics. 

When the VC-dimension of the $D$-layer DNN is $d_{VC}$, there exist $2^{d_{VC}}$ realizations of DNNs to shatter $d_{VC}$ samples $\mathbf{x}_1$, ..., $\mathbf{x}_{d_{VC}}$. We can consider a composite vector $\mathbf{z}=(\mathbf{x}_1,...,\mathbf{x}_{d_{VC}})$, and the transformation in (\ref{eq:ensemble}). Note that each individual component $\mathbf{x}_i$ experiences the same transformations. The $2^{d_{VC}}$ sets of transformations of the DNN drive $\mathbf{z}$ to $2^{d_{VC}}$ vectors of distinct signs. Since the output of the $D$ transformations is $\epsilon$-affine separable, they form $2^{d_{VC}}$ distinct paths (separated by a distance of at least $\epsilon$) in the dynamics. According to the definition of ensemble entropy, the $2^{d_{VC}}$ sets of transformations generate at most $2^{\sup_{\{T_{mt}\}_{m,t},\mathbf{z}}H_e(\{T_{mt}\}_{m=1,...,2^{d_{VC}},t=1,...,D},D,\epsilon,\mathbf{z})D}$ distinct trajectories. Therefore, we have
\begin{eqnarray}
2^d_{VC}\leq 2^{\sup_{\{T_{mt}\}_{m,t},\mathbf{z}}h_e(\{T_{mt}\}_{m=1,...,2^{d_{VC}},t=1,...,D},D,\epsilon,\mathbf{z})D},
\end{eqnarray}
which concludes the proof.
\end{proof}

\section{Proof of Theorem \ref{thm:vc_lbound}}\label{appdx:vc_lbound} 
We first show intuitive images of shattering different sets of sampls.
One can use the nonlinear boundary to shatter samples that cannot be shattered by linear classifiers. When a new point is added to a set of shattered samples and makes the new set impossible to shatter, the two possible situations are illustrated in Fig. \ref{fig:scattering}, where the new point (of label +1) is either included in the convex hull of the samples of label -1, or not. 

\begin{figure}
  \centering
  \includegraphics[scale=0.3]{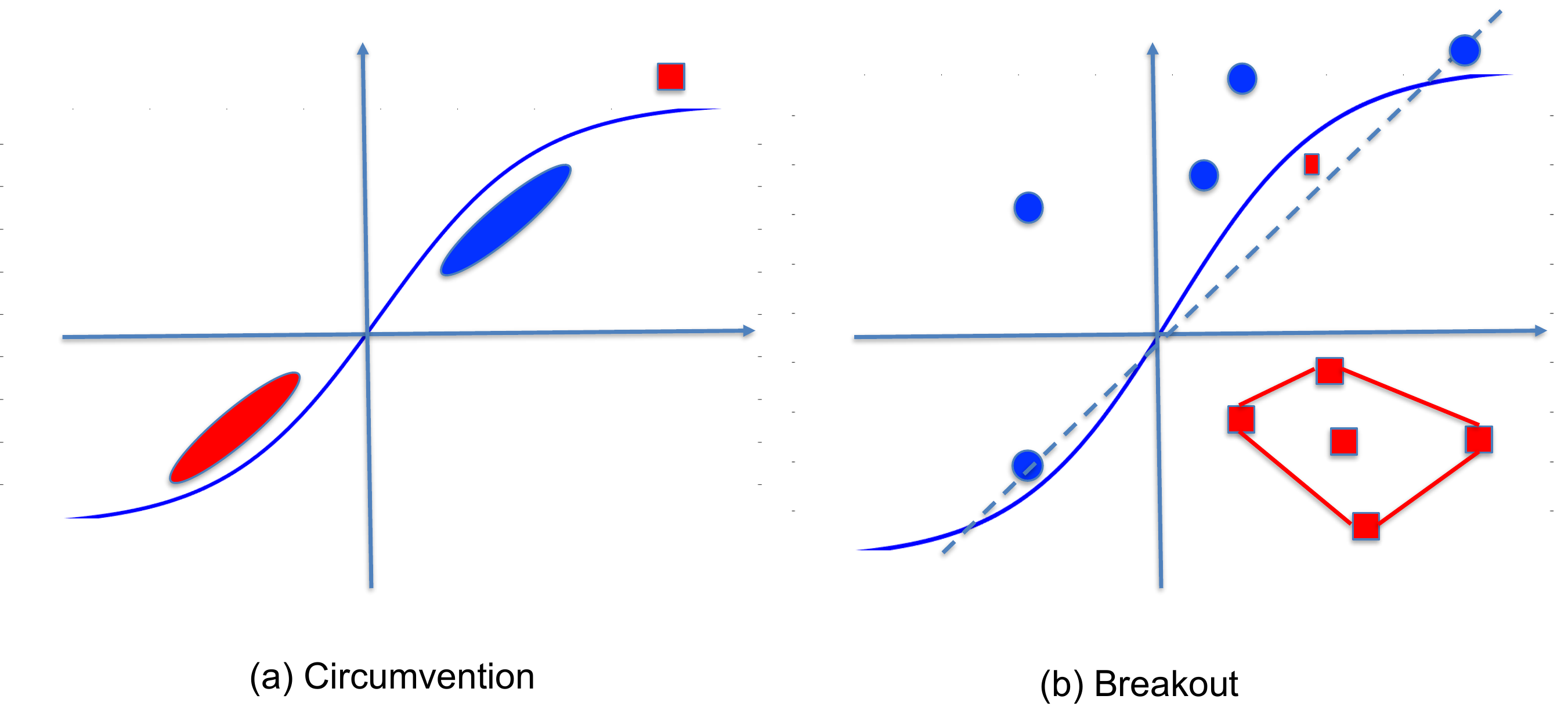}
  \caption{Scattering using the nonlinearity of $\tanh$}\label{fig:scattering}
\end{figure}

We need the following lemma, which is illustrated in Fig. \ref{fig:vector1}:
\begin{lemma}
The mapping $\tanh^{-1}$ maps the line $x_2=ax_1$ ($a<1$) in $[-1,1]^2$ to a curve $x_2=f_a(x_1)$, where $f_a$ is dependent on $a$ and is concave when $x>0$.
\end{lemma}
\begin{remark}
When $a=1$, the mapping $\tanh^{-1}$ maps from line to line. When $a>1$, the corresponding function $f_a$ is convex.
\end{remark}

\begin{proof}
The mapping $\tanh^{-1}$ maps the line $y=ax$, $x\in [-1,1]$, to the following curve:
\begin{eqnarray}\label{eq:curve}
y=\tanh^{-1}(a\tanh(x)).
\end{eqnarray}

It is easy to verify
\begin{eqnarray}
\frac{d^2 y}{dx^2}&=&-\frac{2a(1-\tanh^2(x))\tanh(x)}{1-a^2\tanh^2(x)}\nonumber\\
&\times&\left(1-\frac{a^2(1-\tanh^2(x))}{1-a^2\tanh^2(x)}\right),
\end{eqnarray}
where it is easy to verify $\frac{a^2(1-\tanh^2(x))}{1-a^2\tanh^2(x)}<1$, since $a<1$ Therefore, we have $\frac{d^2 y}{dx^2}<0$, such that the curve in (\ref{eq:curve}) is concave. 
\end{proof}

\begin{figure}
  \centering
  \includegraphics[scale=0.4]{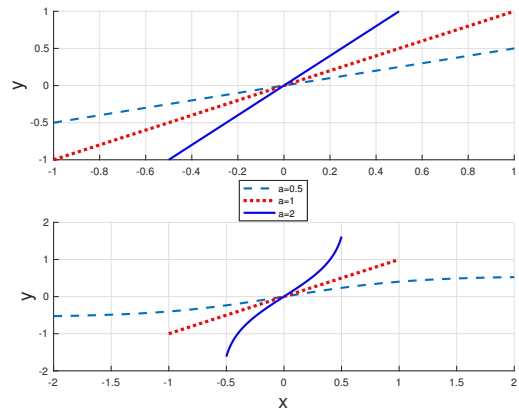}
  \caption{Deformation of lines under $\tanh^{-1}$, where $a>1$.}\label{fig:vector1}
\end{figure}

The following lemma, as illustrated in Fig. \ref{fig:move}, is also needed, which can be easily proved by identifying two boundary points on $A$, $z_1$ and $z_2$, and finding an affine transformation from $P_1$ and $P_2$ to $z_1$ and $z_2$.
\begin{lemma}\label{lem:intersection}
Consider a closed convex set $A$, and two points $P_1$ and $P_2$ on a line $L$. There always exists an affine transformation $T$ such that the intersection points of $T(L)$ and $A$ are exactly $T(P_1)$ and $T(P_2)$.
\end{lemma}

\begin{figure}
  \centering
  \includegraphics[scale=0.4]{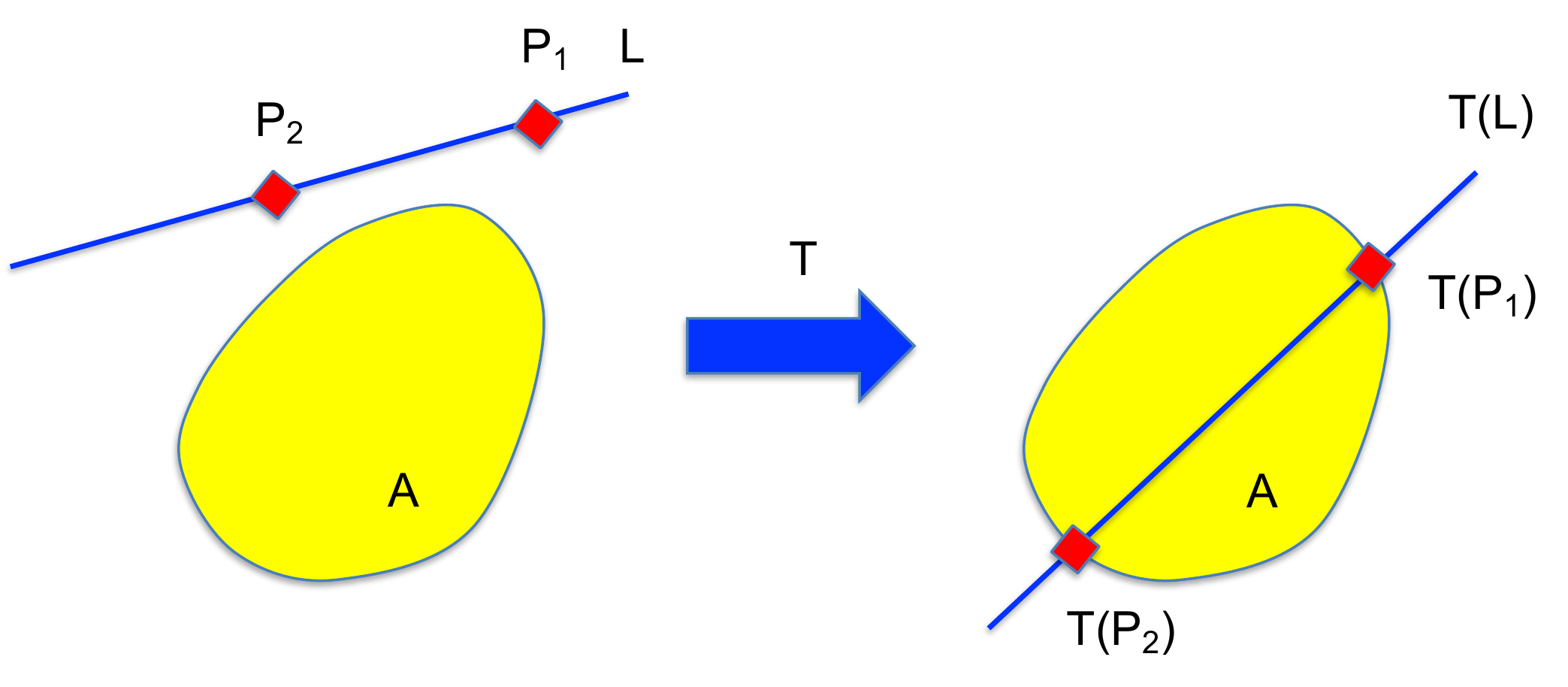}
  \caption{An illustration of Lemma \ref{lem:intersection}.}\label{fig:move}
\end{figure}

Then we prove Theorem \ref{thm:vc_bound} using induction:

\begin{proof}
We apply induction on the number of shattered samples. When there is only one layer in the DNN, it is easy to verify that the single nonlinear mapping $\tanh$ can shatter four points. The procedure is illustrated in Fig. \ref{fig:four}, where a linear classifier cannot label the four points. Then, the intersection point of the two lines $L_1$ and $L_2$ is affinely shifted to the origin. Choose a matrix $\mathbf{A}_1$ sufficiently stretching $L_1$ and compressing $L_2$ by scales $s$ and $\frac{1}{s}$. The intersection points of $L_1$ and $L_2$ is affinely shifted to $(0,1)$. When $s$ is sufficiently large, $\mathbf{x}_1$ and $\mathbf{x}_2$ are mapped by $\tanh$ sufficiently close to $(1,1)$ and $(-1,-1)$, respectively, while $\mathbf{x}_3$ and $\mathbf{x}_4$ are mapped sufficiently close to $(0,\tanh(1))$. Then, the images of the four samples can be linearly classified.

\begin{figure}
  \centering
  \includegraphics[scale=0.4]{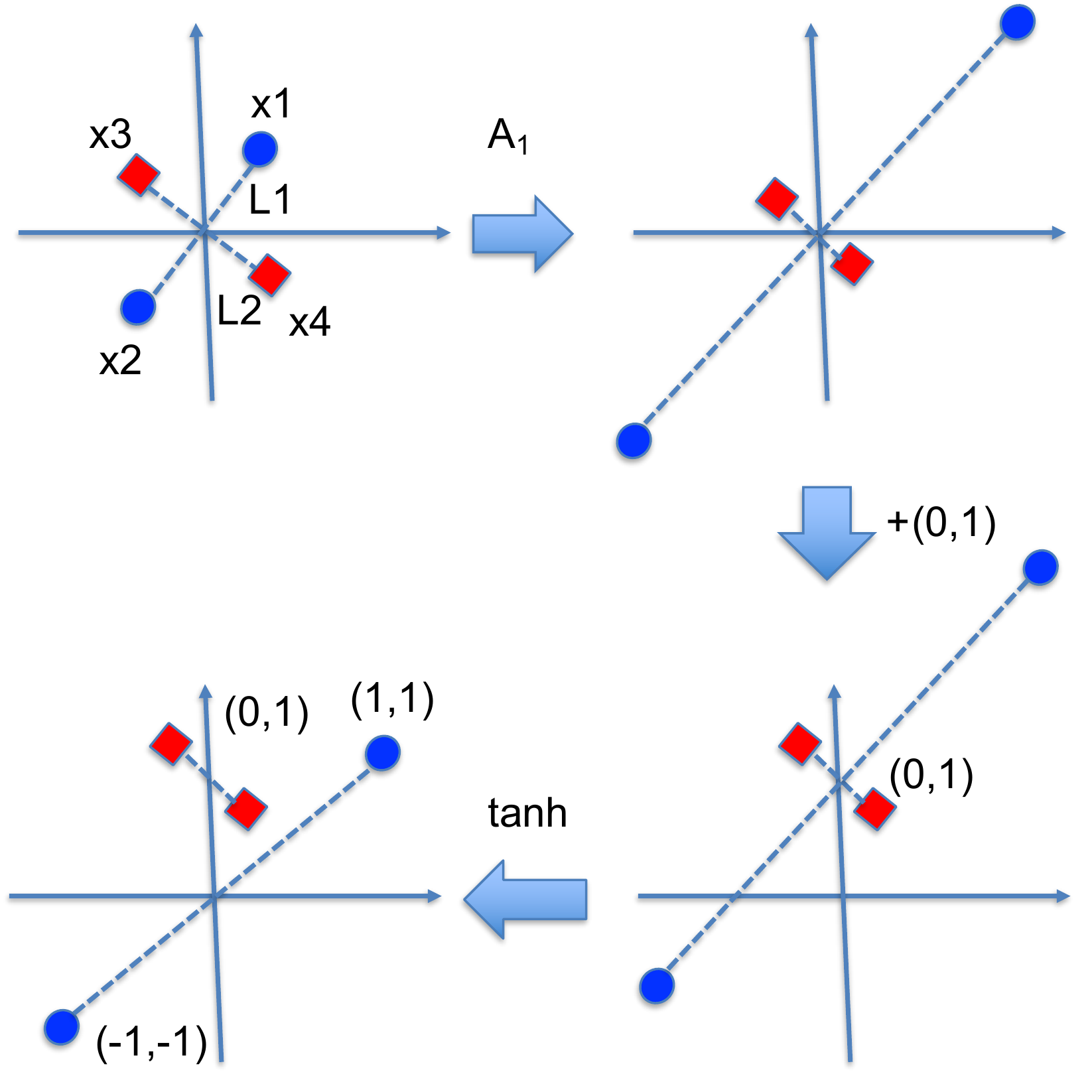}
  \caption{Shattering four samples.}\label{fig:four}
\end{figure}

Then, we assume that $k$ samples ($k\geq 4$) can be shattered when $D=k-3$. We add a new sample $\mathbf{x}_{k+1}$ in $\mathbb{R}^2$. The selection is arbitrary except for the requirement that $\mathbf{x}_{k+1}$ be not within the convex hull of the existing $k$ samples. 

Now, we want to prove that the $k+1$ samples can be shattered by an ensemble of $k-2$-layer DNNs. We assume that there exists a setup of labelling that cannot be achieved by any $k-3$-layer DNN. We assume that, in this setup, $\Omega_{+1}=\{\mathbf{x}_1,...,\mathbf{x}_m,\mathbf{x}_{k+1}\}$ are labeled as +1 and $\Omega_{-1}\{\mathbf{x}_{m+1},...,\mathbf{x}_{k}\}$ are labeled as -1. However, the linear classifier cannot label all samples in $\Omega_{+1}$ as +1. 
Due to the assumption of induction, the set $\mathbf{x}_1,...,\mathbf{x}_k$ can be shattered by $k-3$ layers of DNN. Therefore, after $k-3$ layers of mappings, the convex hulls of $\Omega_{+1}/\mathbf{x}_{k+1}$ and $\Omega_{-1}$ do not intersect, such that there exists a line to separate them. For notational simplicity, we still use $\mathbf{x}_1$, ..., $\mathbf{x}_k$ to represent the images after the $k-3$ layers of mappings.

\begin{figure}
  \centering
  \includegraphics[scale=0.4]{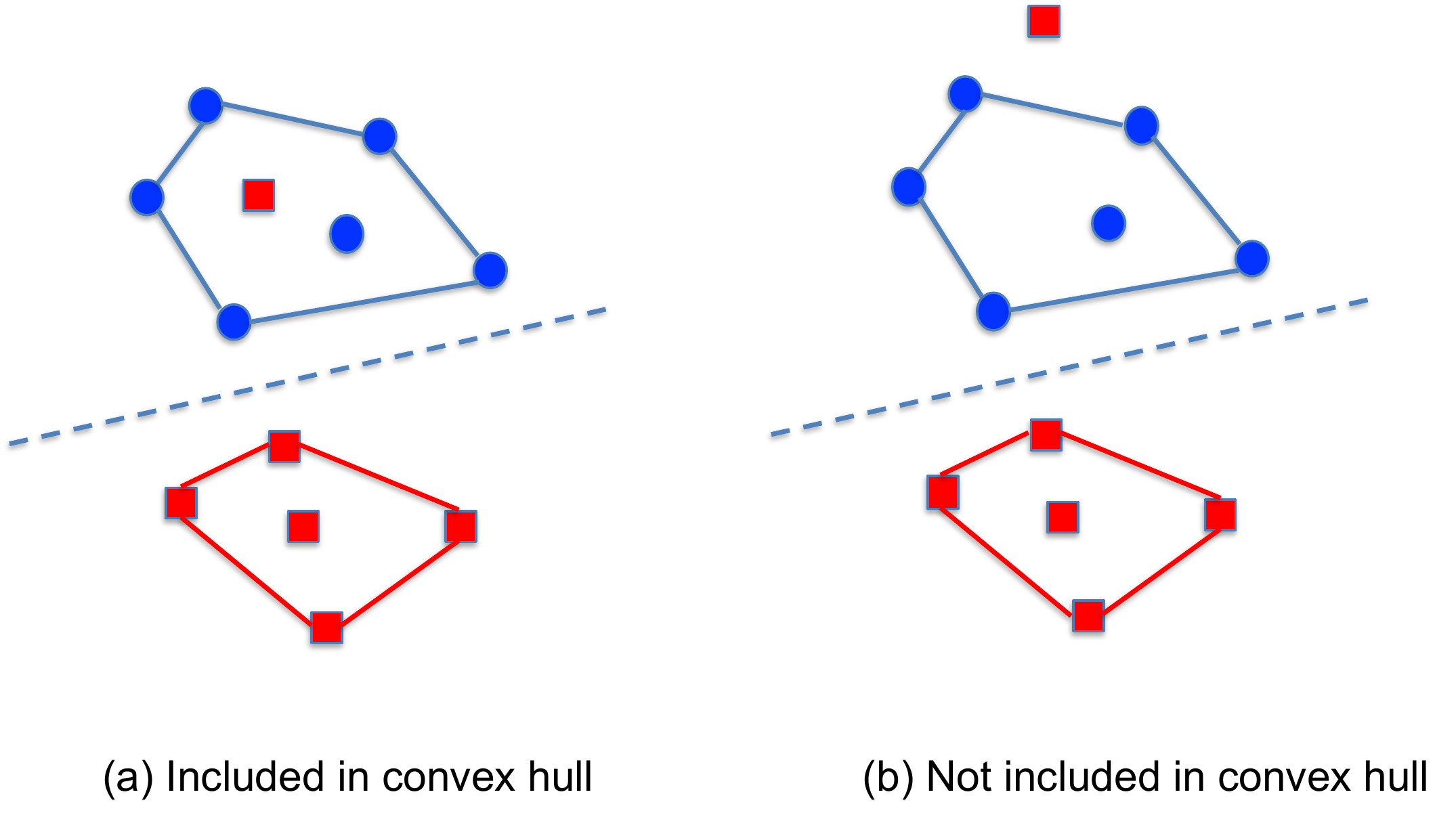}
  \caption{Cases of indistinguishable by linear classifiers in $\mathbb{R}^2$.}\label{fig:impossible}
\end{figure}

There are two possibilities for the impossibility of linearly classifying the $k+1$ samples given the labels, as shown in Fig. \ref{fig:impossible}. Due to the assumption that $\mathbf{x}_{k+1}$ is not within the convex hull of $\{\mathbf{x}_1,...,\mathbf{x}_k\}$, the situation (a) in Fig. \ref{fig:impossible}, where $\mathbf{x}_{k+1}$ is within the convex hull of $\Omega_{-1}$, is impossible. Therefore, we consider only the case in Fig. \ref{fig:impossible} (b), namely the sample $\mathbf{x}_{k+1}$ is not in the convex hull of $\Omega_{-1}$. Then, we consider the straight line $L$, passing $\mathbf{x}_{k+1}$ and the centroid of $\Omega_{+1}/\mathbf{x}_{k+1}$ (namely, $\frac{1}{m}\sum_{j=1}^m \mathbf{x}_j$). The line $L$ intersects $\Omega_{-1}$; otherwise, a slight shift of $L$ can distinguish $\Omega_{+1}$ and $\Omega_{-1}$, thus contradicting the assumption that they cannot be linearly distinguished. 

\begin{figure}
  \centering
  \includegraphics[scale=0.3]{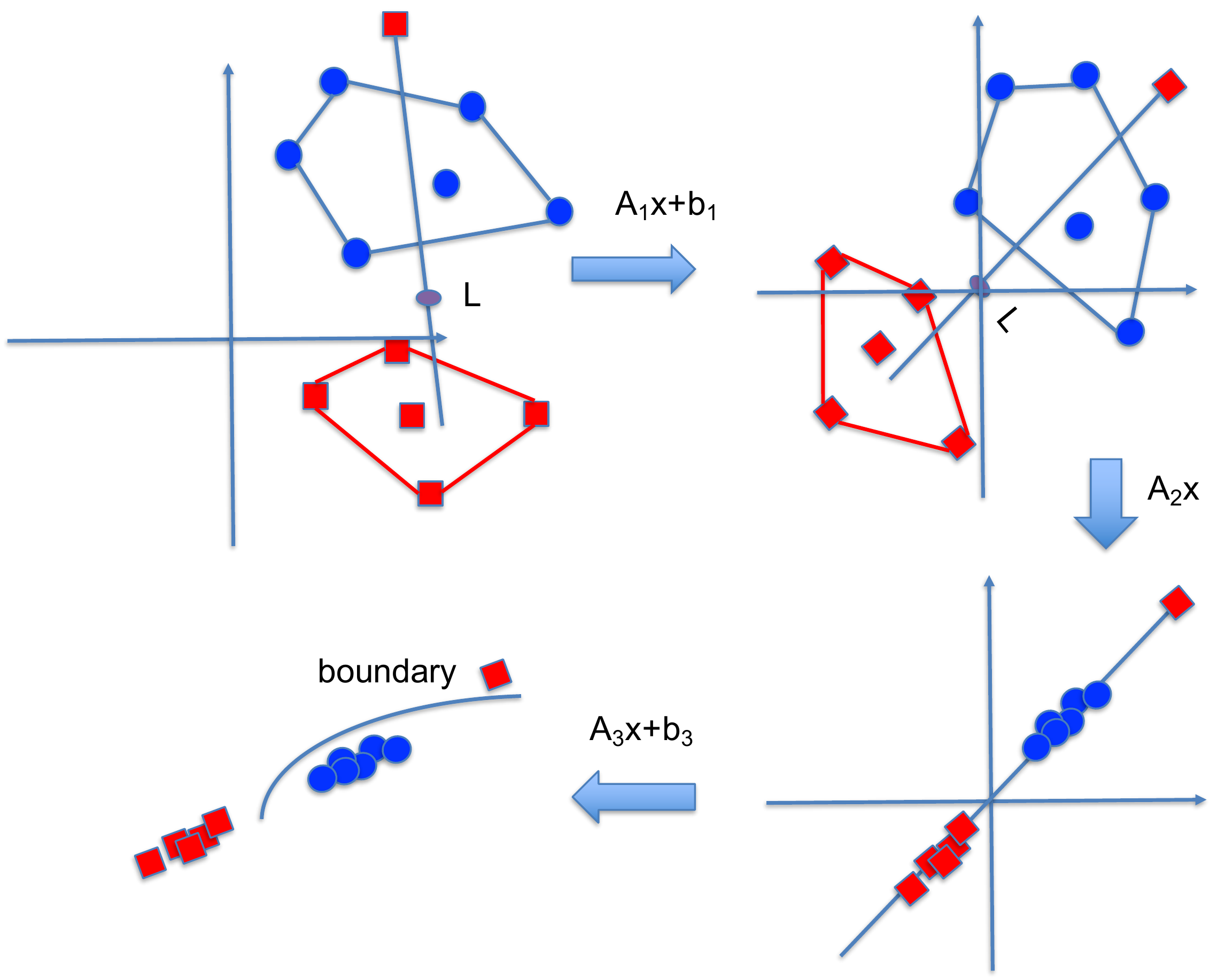}
  \caption{Successive affine mappings for classification.}\label{fig:steps}
\end{figure}

Then, we carry out the following successive affine transformations, which are illustrated in Fig. \ref{fig:steps}:
\begin{itemize}
\item We choose a shift $\mathbf{b}_0$ to change the origin to a point $P$ on $L$, such that $\Omega_{+1}/\{\mathbf{x}_{k+1}\}$ and $\Omega_{-1}\cup \{\mathbf{x}_{k+1}\}$ are on the two sides of $P$. Then, we use a linear transformation $\mathbf{A}_1$ to rotate the line $L$ to the angle of $45^\circ$.

\item We choose a linear transformation $\mathbf{A}_2$, such that one of its eigenvectors $\mathbf{v}_1$ is along the line $L$ while the other eigenvector $\mathbf{v}_2$ is orthogonal to $\mathbf{v}_1$. We select the eigenvalue $\lambda_1$ to be 1 while the other eigenvalue $\lambda_2$ to be sufficiently small. Then, all the samples become sufficiently close to the line $L$. 

\item Choose points $P_1$ and $P_2$ on $L$ such that they separate $\Omega_{+1}/\{\mathbf{x}_{k+1}\}$, $\Omega_{-1}$ and $\mathbf{x}_{k+1}$. Since the image of the line $x_2=ax_1$ ($a<1$) under $\tanh^{-1}$ is a concave function in $\mathbb{R}^2$, the corresponding epigram $E$ is a convex set. We apply Lemma \ref{lem:intersection} such that there exists an affine transformation $T$ characterized by matrix $\mathbf{A}_3$ and vector $\mathbf{b}_3$ such that $\mathbf{A}_3(L)$ intersects $E$ at $\mathbf{A}_3(P_1)$ and $\mathbf{A}_3(P_2)$. Then, we have $\Omega_{+1}\in \bar{E}$ and $\Omega_{-1}\in E$.
\end{itemize}

In the last step, obviously, the sets $\Omega_{+1}$ and $\Omega_{-1}$ are separated by the image of line $x_2=ax_1$ under $\tanh^{-1}$. Then, the images of $\Omega_{+1}$ and $\Omega_{-1}$ under the mapping $\tanh$ are linearly separable. Therefore, by the induction, we have proved that there exist $k$ samples that can be shattered by $k-3$-layer DNNs. Therefore, the VC-dimension is at least $D+3$. This concludes the proof.
\end{proof}

\end{document}